\documentclass[10pt]{article} 
\usepackage[preprint]{tmlr}

\usepackage{amsmath,amsfonts,bm}

\def\eqref#1{equation~\ref{#1}}

\def\1{\bm{1}}

\DeclareMathAlphabet{\mathsfit}{\encodingdefault}{\sfdefault}{m}{sl}
\SetMathAlphabet{\mathsfit}{bold}{\encodingdefault}{\sfdefault}{bx}{n}

\DeclareMathOperator*{\argmax}{arg\,max}
\DeclareMathOperator*{\argmin}{arg\,min}

\usepackage{hyperref}
\usepackage{url}
\usepackage{graphicx}
\usepackage{booktabs}
\usepackage{amsmath}
\usepackage{amssymb}
\usepackage{amsfonts}
\usepackage{amsthm}
\usepackage{float}
\usepackage{multirow}
\usepackage{longtable}
\usepackage{algorithm}
\usepackage{algpseudocode}

\newcommand{\Pred}{\mathrm{Pred}}
\newcommand{\Feat}{\mathrm{Feat}}
\providecommand{\argmin}{\mathop{\mathrm{argmin}}\limits}
\providecommand{\argmax}{\mathop{\mathrm{argmax}}\limits}
\newtheorem{proposition}{Proposition}

\title{Dynamic Multi-Step Decoding for Time-Series Foundation Models}

\author{\name Birmjune Kim\thanks{Equal contribution.} \email birmjunekim@snu.ac.kr \\
      \addr Department of Computer Science and Engineering, Seoul National University
      \AND
      \name Hyungju Lee\footnotemark[1] \email dlgudwn1219@snu.ac.kr \\
      \addr Department of Computer Science and Engineering, Seoul National University
      \AND
      \name Hyun Oh Song \email hyunoh@snu.ac.kr \\
      \addr Department of Computer Science and Engineering, Seoul National University\\
      Interdisciplinary Program in Artificial Intelligence, Seoul National University}

\def\month{10}
\def\year{2026}
\def\openreview{\url{https://openreview.net/forum?id=XXXX}}

\begin{document}

\maketitle

\begin{abstract}
Time-series forecasting models typically use a fixed decoding strategy at inference time: either one-shot prediction of the full horizon or autoregressive rollout over atomic chunks. While one-shot decoding reduces error accumulation, it can discard useful sequential structure; conversely, autoregressive decoding preserves temporal dependencies but may compound errors. We argue that the optimal decoding strategy is not fixed across inputs, but should be selected adaptively for each forecasting instance. We propose a dynamic strategy selector that treats the forecast rollout schedule as an inference-time decision problem for frozen pretrained time-series foundation models. 
The selector is a learned network that ranks candidate rollout schedules using hidden states of the backbone, and is trained on historical windows of the evaluated tasks for each backbone. To make the candidate set tractable over the exponentially large space of horizon partitions, we use dynamic programming that reduces candidate evaluation to quadratic complexity while preserving much of the oracle gain.
At deployment, five independently trained selectors vote on each decision, and the policy follows their consensus, falling back to the backbone's one-shot forecast where they disagree. The construction requires only the backbone's hidden states and lightweight selectors, with no fine-tuning of the forecasting backbone. 
On fev-bench with a frozen Chronos-2 backbone, agreement-gated selection improves point-forecast skill (MASE) over one-shot decoding with statistical significance, mostly on longer horizons. It recovers about 6--8\% of the gain that the same dynamic program attains with ground-truth selection. With a frozen Moirai-1.1-R-small backbone, the improvement is also significant on the full benchmark, and for both backbones we claim gains in point accuracy only.
\end{abstract}

\section{Introduction}

Time-series forecasting, the task of predicting future values from a
historical look-back window, is one of the most fundamental problems in
time series analysis. The field has progressed from classical statistical
methods such as ARIMA and GARCH~\citep{box1970time,bollerslev1986generalized}
to sophisticated deep learning architectures, including temporal
convolutional networks~\citep{bai2018tcn} and Transformer-based
forecasters~\citep{zhou2021informer,wu2021autoformer,zhou2022fedformer}.
Most recently, the paradigm has shifted toward Transformer-based foundation
models such as Chronos-2~\citep{ansari2025chronos2},
TimesFM~\citep{das2024timesfm}, and Moirai~\citep{woo2024moirai}, which leverage large-scale pretraining to
achieve zero-shot generalization. Despite these architectural advances and
the success of pretraining, the impact of inference strategies, specifically
the granularity at which forecasts are decoded, has not received enough
attention.

Traditional autoregressive (AR) decoders generate predictions one step
at a time, feeding each step's output back as input for the next. While
conceptually simple and intuitive, this step-by-step process is prone to
error accumulation, where minor inaccuracies at early steps propagate and
amplify over long horizons. To avoid this recursive error growth, many
recent forecasting models adopt direct multi-step decoding: models such as
Informer, FEDformer, DLinear, iTransformer, TiDE, and PatchTST predict
the forecast horizon, or a large output block, in a single forward
pass~\citep{zhou2021informer,zhou2022fedformer,zeng2023transformers,
liu2024itransformer,das2023tide,nie2023patchtst}. This direct strategy
improves inference efficiency and avoids repeated feedback of its own
predictions, but loses the ability to condition each prediction on
previously generated values. Modern foundation models such as
Chronos-2~\citep{ansari2025chronos2} and TimesFM~\citep{das2024timesfm}
largely follow this direct-decoding paradigm, rolling out only when the
requested horizon exceeds their native maximum output length.

In this paper, we ask whether the decoding schedule should itself
be chosen, rather than fixed at inference time. We posit that the best rollout schedule varies with the input context,
and may further vary with the backbone, so no single fixed schedule
should be expected to be optimal everywhere. What this paper demonstrates is that this
variability can be exploited: dynamically choosing the schedule for
each forecasting instance improves over the backbone's default one-shot
decoding. Because the choice happens at inference time, it can be made for patch-based frozen backbones without modifying their parameters. We evaluate it with two such backbones, Chronos-2 and Moirai-1.1-R-small.
We introduce a framework with two components:
(i) a dynamic-programming search that reduces the exponential space of
decoding schedules to a quadratic number of candidates, and (ii) a
lightweight selector network that ranks these candidates for each input
window using the backbone's hidden states. The selector is trained on
historical windows of the evaluated tasks for a fixed backbone, so the
intended setting is adaptation to tasks whose history is available, not
transfer to unseen tasks (Appendix~\ref{app:full-results}). On fev-bench
with a frozen Chronos-2 backbone, our method improves point-forecast
skill over one-shot decoding by $+0.20\%$ on the full benchmark, with a
bootstrap confidence interval that excludes zero. The gain grows to
$+0.68\%$ on the 19 tasks with the longest horizons. Ground-truth
selection within the same dynamic program attains $+3.25\%$ and
$+9.13\%$ on these two sets, so the selector recovers about 6--8\% of
this headroom. With a frozen Moirai-1.1-R-small backbone, the
improvement is also significant on the full benchmark ($+0.23\%$), and
for both backbones we claim gains in point accuracy only.

\section{Related Work}

\paragraph{Forecasting models.}
A long line of work designs new architectures for time series
forecasting. Transformer variants reduce attention cost or inject
periodicity priors~\citep{zhou2021informer, wu2021autoformer,
zhou2022fedformer}; PatchTST~\citep{nie2023patchtst} demonstrates
that channel-independent patching alone is competitive on long
horizons; DLinear~\citep{zeng2023transformers} matches several
Transformer baselines with a one-layer linear model;
N-BEATS~\citep{oreshkin2020nbeats}, N-HiTS~\citep{challu2023nhits},
and iTransformer~\citep{liu2024itransformer} broaden the design
space further. A more recent line pretrains a single backbone on
heterogeneous corpora for zero-shot transfer:
Chronos~\citep{ansari2024chronos},
TimesFM~\citep{das2024timesfm}, and
Moirai~\citep{woo2024moirai}. These works expand the pool of
frozen backbones available at inference but leave open how each
window should be decoded.

\paragraph{Input-side adaptation.}
Several works modify how the series is presented to the backbone.
PatchTST treats patch size and stride as first-class
hyperparameters~\citep{nie2023patchtst}; Moirai introduces
multiple patch-projection heads for different sampling
frequencies~\citep{woo2024moirai}; Chronos quantizes scaled values
into a fixed bin vocabulary~\citep{ansari2024chronos}; and
WaveToken~\citep{masserano2025wavetoken} replaces bin-based
tokenization with a wavelet-domain vocabulary. Pathformer goes
further with a multi-scale router that weights several patch-size
branches inside a jointly trained
Transformer~\citep{chen2024pathformer}. SRSNet instead questions the
fixed adjacent-patching assumption itself: its SRS module adaptively
selects informative patches and dynamically reassembles them to form a
selective representation space for patch-based forecasters~\citep{wu2025srsnet}.
All of these focus on the \emph{input} side; the decoding procedure
remains a fixed backbone-level choice.

\paragraph{Multi-step decoding and input-conditional routing.}
The trade-off between direct and iterated multi-step prediction
has long been studied in classical
forecasting~\citep{marcellino2006comparison,
chevillon2007direct, taieb2016bias}. Among
pretrained foundation models, the most direct empirical signal we are aware of
comes from TimesFM~\citep{das2024timesfm}: increasing the output
patch length, and thereby reducing the number of autoregressive
steps, monotonically improves long-horizon accuracy on average.
This is, however, a globally tuned hyperparameter, not a
per-instance decision; none of the fixed schedules we evaluate
improves on one-shot decoding for Chronos-2
(Appendix~\ref{app:full-results}).

Multi-step strategies have also been compared and selected outside
pretrained models. DyStrat~\citep{green2024dystrat} trains its own
forecasters for each recursive, direct, or hybrid strategy and selects
among these separately trained models with a classifier on the input
window. Stratify~\citep{green2025stratify} unifies these strategies in
one parameterized space and finds that no strategy dominates across
tasks, in line with our fixed-schedule results; its strategies are
compared per task rather than selected per instance. In both works, each
strategy is realized by forecasters trained from scratch.
Beyond forecasting strategies, a related mechanism is
input-dependent expert routing: sparse mixture-of-experts
gates each input to a subset of experts in
NLP~\citep{shazeer2017outrageously, fedus2022switch} and in time
series via Time-MoE~\citep{shi2025timemoe} and
Moirai-MoE~\citep{liu2025moiraimoe}.

We borrow the input-conditional gating view, but route a different
object: instead of choosing which expert processes the input, our method
chooses how a fixed forecasting model should decode the forecast horizon.
Unlike DyStrat and Stratify, it does not train a forecaster per
strategy: the candidates are rollout schedules of one frozen pretrained
model, including chunks of unequal size, and the selector reads the
backbone's hidden states for each candidate. A choice per task is not
enough in our setting: a fixed schedule chosen per task on historical
windows does not improve on one-shot decoding
(Appendix~\ref{app:full-results}). 

Our main contribution is a dynamic programming framework that makes this
rollout-selection problem tractable: although the space of horizon
partitions is exponential, candidate evaluation reduces to quadratic
complexity. To our knowledge,
no prior work selects among patch-level rollout schedules of a single
frozen pretrained forecaster for each input window at inference time.
Finally, our deployment mechanism draws on two standard tools:
seed-replica ensembles, whose disagreement is a classical estimate of
predictive uncertainty~\citep{lakshminarayanan2017simple}, and
selective prediction, where a model abstains when its confidence is
low~\citep{elyaniv2010foundations, geifman2017selective}. Our agreement gate instantiates both,
with the particularity that abstention has a natural fallback: the
backbone's own one-shot forecast, which makes deviation strictly
opt-in.

\section{Preliminary}

\subsection{Notations}
\label{sec:notation}

The time series forecasting task takes as input the past $L$ timesteps of
a $C$-channel series and predicts the following $h$ timesteps, the
\emph{horizon}. The forecasting models we consider emit predictions in
\emph{patches} of $P$ timesteps ($P{=}16$ for Chronos-2, and for
Moirai-1.1-R-small as configured in Section~\ref{sec:experiments}); we
therefore measure the horizon in patches throughout, and write
$H = \lceil h/P \rceil$ for the number of patches it spans. When $P$ does
not divide $h$, the last patch is truncated to its first $h-(H-1)P$
timesteps: the model call that reaches the end of the horizon requests
only the remaining timesteps, and errors are computed on the first $h$
timesteps. To keep the notation light, we write sizes as if every patch
had $P$ timesteps, so that $HP$ stands for $h$ in that case.

We denote the past context by $\mathbf{X} \in \mathbb{R}^{L \times C}$ and the
ground-truth horizon by $\mathbf{Y} \in \mathbb{R}^{HP \times C}$. A
forecasting model $g_{\phi}$ takes a context and a requested number of
patches, and produces a prediction
$g_{\phi}(\mathbf{X}, H) \in \mathbb{R}^{HP \times C}$; the context may be
longer than $L$ timesteps, since the rollouts below append generated
patches to $\mathbf{X}$. For $K \le H$ we write
$\mathbf{Y}_{1:K} \in \mathbb{R}^{KP \times C}$ for the first $K$ patches of
$\mathbf{Y}$, its \emph{prefix} of length $K$, and likewise for predictions.
We measure forecast quality with an error functional
$\ell(\hat{\mathbf{Y}}, \mathbf{Y}) \in \mathbb{R}_{\geq 0}$, defined for any
pair of equal-length prediction and target.

Our method builds the forecast with a dynamic program over the prefix
lengths $K = 1, \ldots, H$ (Section~\ref{sec:dp}). At each $K$ it compares
$K$ candidate prefixes, indexed by $i \in \{0, \ldots, K-1\}$: candidate $i$
reuses the length-$i$ prefix kept at an earlier step and generates the
remaining $K-i$ patches in one model call, so that candidate $0$ generates
all $K$ patches at once.

We assume $g_\phi$ is a foundation Transformer model whose encoder
produces one hidden vector of size $d_m$ per output patch and channel,
$\text{TransformerEncoder}(\mathbf{X}, H) \in \mathbb{R}^{H \times C
\times d_m}$, from which a prediction head emits the forecast:
\[
  g_\phi(\mathbf{X}, H)  = \text{MLP}_{\text{head}}(\text{TransformerEncoder}(\mathbf{X}, H)).
\]
Averaging the hidden states over channels yields a per-patch summary
\[
  \mathbf{E}(\mathbf{X}, H) = \text{AvgPool}_{\text{channel}} \left( \text{TransformerEncoder}(\mathbf{X}, H) \right) \in \mathbb{R}^{H \times d_{m}},
\]
which we later use as the feature representation of generated patches:
Section~\ref{sec:dp} stacks these summaries into a feature sequence
$F_K^{(i)} \in \mathbb{R}^{K \times d_m}$ for each candidate. We write
$[\mathbf{U}, \mathbf{V}]$ for concatenation along the time axis, so that
$\mathbf{U} \in \mathbb{R}^{a \times C}$ and $\mathbf{V} \in \mathbb{R}^{b \times C}$ imply
$[\mathbf{U}, \mathbf{V}] \in \mathbb{R}^{(a+b) \times C}$; for feature
sequences, whose rows are patches, it concatenates along the patch axis in
the same way. Table~\ref{tab:notation} summarizes the notation, including
the symbols defined in Sections~\ref{sec:rollouts} to~\ref{sec:scorer}.

\begin{table}[H]
\caption{Notation.}
\label{tab:notation}
\begin{center}
\small
\begin{tabular}{lll}
\toprule
Symbol & Meaning & Section \\
\midrule
$L$, $h$ & context and horizon length, in timesteps & \ref{sec:notation} \\
$C$ & number of channels & \ref{sec:notation} \\
$P$ & patch size, in timesteps ($16$) & \ref{sec:notation} \\
$H$ & number of horizon patches, $\lceil h/P \rceil$ & \ref{sec:notation} \\
$\mathbf{X}$, $\mathbf{Y}$ & context ($L \times C$) and ground-truth horizon ($HP \times C$) & \ref{sec:notation} \\
$g_\phi$ & frozen forecasting model (backbone) & \ref{sec:notation} \\
$\ell$ & error functional between equal-length prediction and target & \ref{sec:notation} \\
$K$ & prefix length, in patches ($1 \le K \le H$) & \ref{sec:notation} \\
$i$ & candidate index at prefix length $K$ ($0 \le i < K$): length of the reused prefix & \ref{sec:notation} \\
$d_m$ & hidden size of the backbone & \ref{sec:notation} \\
$\mathbf{E}$ & channel-averaged hidden states, one $d_m$-vector per output patch & \ref{sec:notation} \\
$\mathbf{r} = [r_1, \ldots, r_k]$ & rollout: patches produced by each of its $k$ model calls & \ref{sec:rollouts} \\
$\Pred_K$, $\Feat_K$ & prediction and feature prefix retained at length $K$ & \ref{sec:dp} \\
$P_K^{(i)}$, $F_K^{(i)}$ & prediction ($KP \times C$) and features ($K \times d_m$) of candidate $i$ & \ref{sec:dp} \\
$S_\psi$ & learned scorer over the $K$ candidates of a DP row & \ref{sec:scorer} \\
\bottomrule
\end{tabular}
\end{center}
\end{table}

\subsection{Rollouts}
\label{sec:rollouts}

A \emph{rollout} describes how the $H$ horizon patches are produced. Two examples are:
\begin{itemize}
\item \textbf{One Shot}: all $H$ patches are predicted in one forward pass, $g_\phi(\mathbf{X},\, H)$.
\item \textbf{Fully Autoregressive}: one patch is predicted at a time; the model is called $H$ times. 
\[ y_1 = g_{\phi}(\mathbf{X},\, 1), \quad y_2 = g_{\phi}([\mathbf{X},\, y_1],\, 1), \quad \ldots, \quad y_H = g_{\phi}([\mathbf{X},\, y_1,\, y_2, \ldots, y_{H-1}],\, 1) \]
\end{itemize}

Each rollout $\mathbf{r}$ maps to a unique integer decomposition (i.e. an ordered partition) of $H$: $\mathbf{r} = [r_1, r_2, \ldots,r_k]$ such that $\sum_{i=1}^{k} r_i = H$. Here, $r_i \in \mathbb{N}$ denotes the number of patches predicted in the $i$-th call. One-shot rollout corresponds to $[H]$, while full-AR rollout corresponds to $[1, 1, \ldots, 1]$.

The total number of possible rollout strategies is $2^{H-1}$. For a given rollout strategy $\mathbf{r}$, the resulting prediction $g_{\phi}(\mathbf{X})^{(\mathbf{r})} \in \mathbb{R}^{HP \times C}$ is constructed by concatenating the outputs of $k$ sequential model calls:
\[
    \mathbf{y}^{(i)} = g_{\phi}\left( [\mathbf{X}, \mathbf{y}^{(1)}, \dots, \mathbf{y}^{(i-1)}], r_i \right) \quad \text{for } i = 1, \dots, k
\]
where each $\mathbf{y}^{(i)} \in \mathbb{R}^{r_i P \times C}$ is the prediction for the $i$-th segment of the rollout.

\section{Method}

We formulate decoding rollout strategies as the problem of choosing the best rollout 
$\mathbf{r}$ for the current input and forecasting model. Two obstacles
make this challenging. First, the rollout space contains $2^{H-1}$
elements, so finding the optimal rollout by brute force requires
$\mathcal{O}(2^{H})$ search. Second, the ground-truth horizon
$\mathbf{Y}$ needed to score candidate rollouts is unavailable at
inference time. We address these obstacles in two steps.

Section~\ref{sec:dp} replaces exhaustive search over all rollout histories
with a DP-tractable surrogate that evaluates only $\mathcal{O}(H^2)$
candidates. Section~\ref{sec:scorer} then makes this surrogate executable at
inference time by replacing the ground-truth-based scores with a
learned scorer.
 
\subsection{From the original objective to a DP-tractable surrogate}
\label{sec:dp}

\paragraph{The original problem.}
The rollout selector should ideally solve the following exact problem:
\begin{equation}
\operatorname{minimize}_{\mathbf{r}} \;
\ell\bigl( g_{\phi}(\mathbf{X})^{(\mathbf{r})},\, \mathbf{Y} \bigr)
\quad \text{subject to} \quad \mathbf{r} \in \mathcal{P}(H).
\label{eq:original-objective}
\end{equation}
where $\mathcal{P}(H)$ denotes the set of ordered integer partitions of $H$.

Solving this exact problem directly is intractable for two
reasons: the search space has $2^{H-1}$ elements, and at inference time
$\mathbf{Y}$ is unobserved. We first address the search problem by
introducing a DP-tractable surrogate, and later replace its
scoring rule with a learned scorer in Section~\ref{sec:scorer}.

\paragraph{Surrogate problem.}

Consider two rollouts whose final $t$ calls have the same lengths:
\[
  \mathbf{r}^{(1)} = [r_1^{(1)}, \dots, r_a^{(1)}], \quad
  \mathbf{r}^{(2)} = [r_1^{(2)}, \dots, r_b^{(2)}],
  \qquad
  r_{a-i}^{(1)} = r_{b-i}^{(2)} \quad (i = 0, 1, \dots, t-1),
\]
and let $T = \sum_{i=0}^{t-1} r_{a-i}^{(1)}$ be the number of patches
these shared calls produce. Since both rollouts decode the final $T$
patches with the same schedule, we take the quality of the preceding
$(H-T)$ patches as the criterion for choosing between them, a working
approximation that we quantify empirically below.
We define this initial sequence of $(H-T)$ patches as the \textit{prefix} of a rollout. 
Consequently, we replace the full-horizon objective with a prefix-level
surrogate: at each prefix length \(K\), we retain the candidate with the
lowest error against the ground-truth prefix \(\mathbf{Y}_{1:K}\). This
reduces the search to recursively maintaining one retained prefix for
each length \(K\), denoted by \(\text{Pred}_K\). 

We call this scheme the \emph{prefix-retaining surrogate}: instead of exploring every rollout
history, it retains exactly one prefix per length, and only rollouts that
pass through a retained prefix are ever considered. Equivalently, it is a
beam search over rollout histories with beam width one per prefix length.

We initialize $\text{Pred}_0=\emptyset$. For each prefix length
$K=1,\ldots,H$, the surrogate forms $K$ candidates, one for each
predecessor length $i<K$:
\[
  P_K^{(i)}
  =
  \left[
    \text{Pred}_i,\;
    g_{\phi}\!\left([\mathbf{X},\text{Pred}_i],\,K-i\right)
  \right],
  \qquad 0 \le i < K .
\]
Here, $P_K^{(i)}$ is the length-$K$ prediction obtained by extending
the retained prefix $\text{Pred}_i$ with one additional backbone call of
length $K-i$. For each prefix length $K$, the ground-truth-based rule retains the
candidate whose prediction prefix is closest to the ground-truth prefix:
\[
i_K^{\star} = \argmin_{0 \le i < K} \ell\bigl( P_K^{(i)},\, \mathbf{Y}_{1:K} \bigr)
\]
The selected candidate is then retained as the prefix for length $K$:
\[
  \text{Pred}_K \gets P_K^{(i_K^{\star})}.
\]
Thus, the surrogate replaces the exponential space of rollouts with
$\sum_{K=1}^{H} K = \mathcal{O}(H^2)$ candidate evaluations.
Figure~\ref{fig:dp-tree} illustrates the retained and discarded rollout
histories at $H{=}4$.

\begin{figure}[H]
\begin{center}
\includegraphics[width=0.8\linewidth]{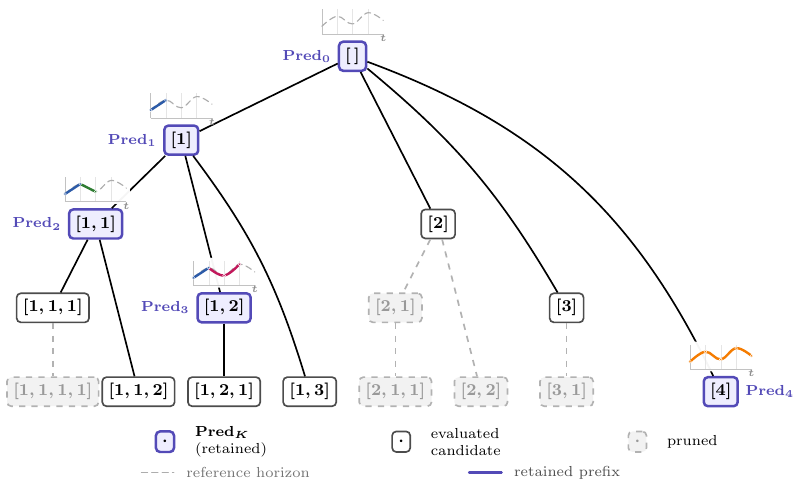}
\end{center}
\caption{
Prefix-retaining surrogate search at $H{=}4$. Among the
$2^{H-1}{=}8$ possible rollout histories, the surrogate retains a
single prediction prefix at each horizon-prefix length. Highlighted
nodes denote the retained predictions $\text{Pred}_K$, while gray dashed
nodes indicate rollout histories discarded because they would require a
different retained prefix at an earlier length. The node text shows the
corresponding rollout decomposition for visualization only, and the
small curves above highlighted nodes show the actual retained
time-series prefix.
}

\label{fig:dp-tree}
\end{figure}

\paragraph{Experimental evidence.}
The prefix-retaining surrogate is strictly weaker than the original
problem: each step retains a single candidate on the basis of
length-$K$ prefix error alone, which can discard rollout histories that
are suboptimal at an intermediate prefix but optimal at the final
horizon. Therefore, the set of rollout histories retained by the
surrogate is a strict subset of the full $2^{H-1}$ rollout space.

However, this compact candidate set empirically suffices to retain most of the performance gain. The surrogate
evaluates $H(H{+}1)/2$ candidates rather than the $2^{H-1}$ possible
rollouts ($66$ of $1024$ at $H{=}11$). Yet with Chronos-2, on the official evaluation
windows of the 58 tasks with $2 \le H \le 11$ patches, it captures on
average $96.7\%$ of the MASE reduction that exhaustive search achieves
over one-shot decoding ($93.2\%$ on the 28 tasks with $H \ge 4$).
Retaining one prefix per horizon length thus preserves most of the
achievable gain while making selection quadratic.

\paragraph{DP features.}
The retained prediction prefix $\text{Pred}_K$ is sufficient for generating future candidate forecasts, since every later candidate is formed by extending some retained prefix $\text{Pred}_i$. For scoring candidates, however, the predicted values are not the only signal the frozen backbone offers: its hidden states summarize how each prefix was produced, and carry information about the forecast that the values alone may not expose. To give the learned scorer access to this richer signal, we additionally cache a feature prefix $\Feat_K$ alongside each retained prediction prefix.

We initialize $\Feat_0=\emptyset$. For each candidate
$P_K^{(i)}$ from the surrogate search, we construct the corresponding
feature candidate
\[
  F_K^{(i)}
  =
  \left[
    \Feat_i,\;
    \mathbf{E}\!\left([\mathbf{X},\text{Pred}_i],\,K-i\right)
  \right],
  \qquad 0 \le i < K .
\]
Here, $\mathbf{E}([\mathbf{X},\text{Pred}_i],K-i)$ denotes the hidden
representations associated with the newly generated suffix of length
$K-i$. The feature prefix $\Feat_K$ is aligned with the same
length-$K$ horizon prefix as $\text{Pred}_K$: it caches the hidden
representations of the generated output patches retained by the
surrogate. The same lowest-error predecessor selected for the prediction prefix is
used to retain the feature prefix:
\[
  \Feat_K \gets F_K^{(i_K^{\star})}.
\]
Thus, the state retained at prefix length $K$ is the pair
\[
  (\text{Pred}_K,\Feat_K)
  =
  \left(
    P_K^{(i_K^{\star})},
    F_K^{(i_K^{\star})}
  \right).
\]
The prediction component $\text{Pred}_K$ is used to generate future
candidates, while the feature component $\Feat_K$ is used only by the
learned scorer.

\paragraph{DP algorithm.}
With the candidate prediction and feature prefixes defined above,
Algorithm~\ref{alg:prefix-retaining-dp} implements the prefix-retaining
surrogate when ground-truth prefixes are available. At each prefix length $K$, it forms
$K$ candidates by extending each previously retained prefix
$(\text{Pred}_i,\Feat_i)$ for $i<K$. It then scores each candidate by its
prefix error against the ground-truth prefix $\mathbf{Y}_{1:K}$ and
retains the lowest-error candidate as $(\text{Pred}_K,\Feat_K)$;
Figure~\ref{fig:dp-table-h4-selector-arrow} visualizes this row-by-row
selection for the $H{=}4$ example. The only step
that depends on the unavailable future is the ground-truth-based selection in
line~\ref{line:oracle-argmin}; in \S\ref{sec:scorer}, this row-wise
$\argmin$ is replaced by a learned scorer so that the same DP
procedure can be run at inference time.

\begin{algorithm}[H]
\caption{Prefix-retaining DP with ground-truth selection}
\label{alg:prefix-retaining-dp}
\begin{algorithmic}[1]
\Require context $\mathbf{X} \in \mathbb{R}^{L \times C}$, frozen backbone $g_\phi$, ground-truth horizon $\mathbf{Y} \in \mathbb{R}^{HP \times C}$
\Ensure retained horizon prediction $\text{Pred}_H \in \mathbb{R}^{HP \times C}$
\State $\text{Pred}_0 \gets \emptyset$,\quad $\Feat_0 \gets \emptyset$
\For{$K = 1, \dots, H$}
  \For{$i = 0, \dots, K-1$}
    \State $P_K^{(i)} \gets \left[\text{Pred}_i,\; g_\phi\!\left([\mathbf{X}, \text{Pred}_i], K-i\right)\right]$
    \State $F_K^{(i)} \gets \left[\Feat_i,\; \mathbf{E}\!\left([\mathbf{X}, \text{Pred}_i], K-i\right)\right]$
  \EndFor
  \State $i_K^{\star} \gets
    \argmin_{0 \le i < K}\;
    \ell\bigl( P_K^{(i)},\, \mathbf{Y}_{1:K} \bigr)$
    \label{line:oracle-argmin}
  \State $\text{Pred}_K \gets P_K^{(i_K^{\star})}$
  \State $\Feat_K \gets F_K^{(i_K^{\star})}$
\EndFor
\State \Return $\text{Pred}_H$
\end{algorithmic}
\end{algorithm}

\paragraph{Correctness of the DP implementation.}

\begin{proposition}
\label{prop:dp-correct}
For every $K = 0, \ldots, H$, Algorithm~\ref{alg:prefix-retaining-dp}
computes exactly the retained pair $(\Pred_K, \Feat_K)$ of the
prefix-retaining surrogate.
\end{proposition}

\begin{proof}
We prove by induction on the prefix length \(K\). For \(K=0\), the claim
holds by initialization, since the algorithm sets
\((\Pred_0,\Feat_0)=(\emptyset,\emptyset)\).

Assume that for all prefix lengths \(i<K\), the retained pair
\((\Pred_i,\Feat_i)\) equals the retained prediction and feature prefix
defined by the surrogate. At prefix length \(K\), the algorithm forms,
for every predecessor length \(0 \le i < K\), the candidate
\[
  P_K^{(i)}
  =
  \left[
    \Pred_i,\;
    g_\phi\!\left([\mathbf{X},\Pred_i],K-i\right)
  \right]
\]
and the corresponding feature candidate
\[
  F_K^{(i)}
  =
  \left[
    \Feat_i,\;
    \mathbf{E}\!\left([\mathbf{X},\Pred_i],K-i\right)
  \right].
\]
By the induction hypothesis, these are exactly the candidates obtained
by extending the retained prefixes available to the surrogate. The
algorithm then selects
\[
  i_K^\star
  =
  \argmin_{0 \le i < K}
  \ell\bigl( P_K^{(i)},\, \mathbf{Y}_{1:K} \bigr),
\]
which is the same ground-truth prefix-error rule used in the surrogate
definition. It then updates
\[
  \Pred_K \gets P_K^{(i_K^\star)},
  \qquad
  \Feat_K \gets F_K^{(i_K^\star)}.
\]
Thus the retained pair \((\Pred_K,\Feat_K)\) after row \(K\) is exactly
the retained pair specified by the prefix-retaining surrogate. By
induction, this holds for every \(K=0,\ldots,H\), and the returned
\(\Pred_H\) is the surrogate's retained horizon prediction.
\end{proof}

\begin{figure}[htbp]
\begin{center}
\includegraphics[width=0.92\linewidth]{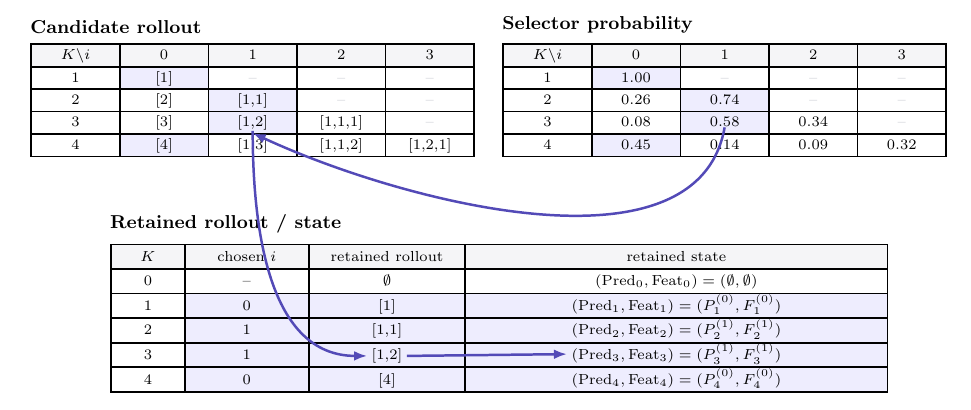}
\end{center}
\caption{Prefix-retaining DP table for the $H{=}4$ example. Each row enumerates candidate predecessors, and the arrow marks the retained choice that commits the next prefix.}
\label{fig:dp-table-h4-selector-arrow}
\end{figure}

\paragraph{Complexity.}
The prefix-retaining surrogate evaluates $\mathcal{O}(H^2)$ candidates
in total, rather than the $\mathcal{O}(2^H)$ rollout histories required
by full enumeration.
The $K$ candidates within a row can be batched in parallel, but rows
must be processed sequentially because row $K+1$ depends on the
computed value from previous rows. Thus, the total compute is
quadratic in $H$, while the sequential depth is $H$. At deployment the
same DP runs with the learned scorer, so the adaptive policy issues
$H(H{+}1)/2$ backbone calls where one-shot decoding issues one ($3$ at
$H{=}2$, $66$ at $H{=}11$, $171$ at $H{=}18$). This cost arises only on
tasks that admit a decision ($H\ge2$), and the scorer's own overhead is
negligible (Section~\ref{sec:ensemble}).

\subsection{Learned scorer for inference}
\label{sec:scorer}

\paragraph{Inference-time selection.}
Algorithm~\ref{alg:prefix-retaining-dp} uses the ground-truth horizon
$\mathbf{Y}$ in line~\ref{line:oracle-argmin} to select the retained
candidate at each prefix length $K$. However, this ground-truth based selection is not
available at inference time. We therefore replace the prefix-error criterion with a learned scorer
\[
  S_\psi :
  \underbrace{
  \mathbb{R}^{K \times d_m}
  \times \cdots \times
  \mathbb{R}^{K \times d_m}
  }_{K\ \text{candidate feature prefixes}}
  \to [0,1]^K,
\]
which takes the \(K\) candidate feature prefixes constructed by the same
DP procedure and returns one score for each candidate. 

At prefix length $K$, the learned scorer jointly scores all candidate
feature prefixes and selects the highest-scoring predecessor:
\[
  i_K^\star
  =
  \argmax_{0 \le i < K}
  \left[
    S_\psi\!\left(
      F_K^{(0)}, F_K^{(1)}, \ldots, F_K^{(K-1)}
    \right)
  \right]_i .
\]
The retained prediction and feature prefixes are then updated as
\[
  (\text{Pred}_K,\Feat_K)
  \gets
  \left(
    P_K^{(i_K^\star)},
    F_K^{(i_K^\star)}
  \right).
\]

\paragraph{Training objective.}
During training, windows are drawn from the historical,
non-evaluation part of each series, so the ground-truth horizon
\(\mathbf{Y}\) is available. For each DP row, we construct the same
candidate prefixes as in inference and compute each candidate's prefix
error against the corresponding ground-truth. The scorer is trained to assign
higher probability to candidates with lower prefix error. 

Specifically, within each DP row we turn the candidates' prefix errors
into a soft target distribution: a softmax over the negated errors,
normalized by the task's error scale and sharpened by a label
temperature. We train the scorer by minimizing the cross-entropy
between its softmax output and this target, averaged over all DP rows
and training windows. In implementation, padded candidate slots are
masked out.

\paragraph{Forecasting model specific scorer.}
In this work, we train a separate scorer for each frozen forecasting
model, Chronos-2 and Moirai-1.1-R-small; conditioning a shared scorer on
a model identifier is an alternative that we do not evaluate. This is
useful because the best retained predecessor can depend on the
forecasting model, even for the same input window.

\paragraph{Architecture.}
At prefix length $K$, $S_\psi$ ingests the $K$ candidate feature
sequences
\[
  F_K^{(0)}, F_K^{(1)}, \ldots, F_K^{(K-1)}
\]
and emits one scalar score per candidate.

\textit{Stage 1 (per-candidate encoding).}
Each candidate feature sequence $F_K^{(i)}$ is projected to width
$d$, augmented with patch-position and role embeddings, prepended
with a learned \texttt{[CLS]} token, and processed by two transformer
encoder blocks. The \texttt{[CLS]} output is used as the candidate
vector.

\textit{Stage 2 (cross-candidate scoring).}
The $K$ candidate vectors are stacked, augmented with candidate-index
and prefix-length embeddings, and processed by two additional
transformer encoder blocks. This allows each candidate to be scored
relative to the other candidates in the same DP row. A linear head then
maps each contextualized candidate vector to a scalar score.

A single set of scorer parameters $\psi$ is shared across all prefix
lengths $K$. The variable number of candidates is handled by the
variable input length of the cross-candidate transformer and by the
learned prefix-length embedding.

\subsection{Agreement-gated ensemble deployment}
\label{sec:ensemble}

A single scorer decides \emph{which} candidate ranks first, but we observe empirically that its score margin is an unreliable guide to \emph{whether} deviating from the one-shot forecast will help.
Agreement across independently trained scorers, which is the classical
uncertainty signal of deep
ensembles~\citep{lakshminarayanan2017simple}, carries this
information: in our validation diagnostics it separates helpful from
harmful deviations.
We therefore train five scorers that differ only in the random seed and let them vote: the
policy adopts a candidate when at least four of the five scorers
propose it and it is not the one-shot continuation, and otherwise keeps
the backbone's one-shot forecast, a form of selective prediction with a
cost-free fallback~\citep{elyaniv2010foundations, geifman2017selective}.
The threshold was fixed before any evaluation. 
Because the one-shot rollout is always among the candidates,
ground-truth selection is never worse than one-shot decoding. A
learned scorer forfeits this guarantee. The agreement gate does not
restore it; it makes deviations more conservative, and in our experiments
empirically more reliable (Table~\ref{tab:ablation}). The scorers' own cost is negligible against
the shared backbone calls. Appendix~\ref{app:gate} expands on the
margin diagnostics, the guarantee, the pre-fixed threshold, and the
cost accounting.

\section{Experiments}
\label{sec:experiments}

\paragraph{Benchmark and protocol.}
We evaluate on fev-bench~\citep{shchur2025fev}, a benchmark of 100
forecasting tasks across seven domains. We follow the official evaluation
protocol without modification (official evaluation windows, MASE, and
cross-learning batching); our pipeline reproduces the officially published
per-task MASE of one-shot Chronos-2 up to a maximal relative difference of
$1.5\times10^{-7}$. Throughout, the error functional $\ell$ is
instantiated as MASE, both for the prefix errors that supervise the
scorer and for evaluation. The MASE scale factor is computed from the
context window, so it is common to all candidates of a DP row and does
not affect the selection. Rollouts feed the backbone's point forecasts
back as context; how to propagate a predictive distribution through a
multi-step rollout is a design question of its own in this model
family, and our claims are confined to point accuracy.

\paragraph{Backbone and scorer training.}
The main backbone is a frozen Chronos-2~\citep{ansari2025chronos2}. As a
second backbone, we use a frozen Moirai-1.1-R-small~\citep{woo2024moirai}.
To align with our formulation, it uses a fixed patch size of 16, and its
point forecast is the median of 5,000 samples. Scorers are trained
separately for each backbone, only on historical windows that strictly
precede the official evaluation windows of each task. We train five scorers
that differ only in the random seed, and select each checkpoint on a
validation split of the historical windows. Training details are given in
Appendix~\ref{app:training}.

\paragraph{Deployment.}
At every decoding decision, each of the five scorers proposes one candidate
(Section~\ref{sec:ensemble}). The deployed policy deviates from one-shot
decoding only when at least four scorers agree on the same candidate and
that candidate is not the one-shot continuation; otherwise the one-shot
forecast is returned unchanged. The agreement threshold of four was fixed
before evaluation.

\paragraph{Metrics and statistics.}
We compare against one-shot decoding at the task level, pairing the
deployed policy with the one-shot forecasts produced by the same
pipeline. A task is a \emph{tie} exactly when the deployed policy never
deviates on it, in which case its output is identical to one-shot: 34
of the 100 tasks have $H{=}1$, admit no decision, and are structural
ties, and further ties arise when the agreement gate never fires on a
task.

\emph{Win rate} counts a tie as half a win; we report it as a
descriptive statistic only, and confine statistical claims to the
skill confidence intervals below.
\emph{Skill} is defined as $(1-\text{geometric mean of per-task
MASE ratios})$ against one-shot, with ratios clipped to $[10^{-2},10^{2}]$
as in the benchmark, and is reported as a
percentage; we report 95\% confidence intervals from a task-level
bootstrap with 20{,}000 resamples. Results are reported on the full
benchmark and on nested subsets of tasks with $H\ge2$, $H\ge4$, and
$H\ge9$.

\subsection{Main results}
\label{sec:main-results}

Table~\ref{tab:main} reports the deployed selector, whose
hyperparameters were chosen on validation data alone by the validation
performance of the deployed policy itself (Appendix~\ref{app:training});
an alternative configuration and per-task results are reported in
Appendix~\ref{app:full-results}. For Chronos-2, the selector's skill improvement
over one-shot decoding is statistically supported on every nested
subset: the 95\% confidence interval excludes zero on the full
benchmark ($+0.20\%$) and the improvement grows with horizon depth,
reaching $+0.68\%$ on tasks with $H\ge9$, where the policy wins 12 of
the 19 tasks against 4 losses. Overall it deviates from one-shot
decoding on 49 of the 66 applicable tasks and wins 29 of those 49.
With Moirai-1.1-R-small, the improvement is likewise statistically
supported on every nested subset ($+0.23\%$ on the full benchmark and $+0.42\%$
on $H\ge4$), with a smaller gain on $H\ge9$ ($+0.28\%$).

\begin{table}[htbp]
\caption{Task-level comparison of the deployed selector against one-shot
decoding of the same frozen backbone on fev-bench. W/T/L are wins, ties, and losses; a tie
means the deployed policy returned the one-shot forecast unchanged on that
task. Win rate and skill are as defined in
Section~\ref{sec:experiments}. Skill estimates whose 95\% confidence interval excludes zero are in bold.}
\label{tab:main}
\begin{center}
\footnotesize
\setlength{\tabcolsep}{3pt}
\begin{tabular}{llcccc}
\toprule
 & & \textbf{full} & \textbf{$H\ge2$} & \textbf{$H\ge4$} & \textbf{$H\ge9$} \\
 & & (100 tasks) & (66 tasks) & (36 tasks) & (19 tasks) \\
\midrule
\multirow{2}{*}{Chronos-2} & W/T/L (win rate) & 29/51/20 (54.5\%) & 29/17/20 (56.8\%) & 19/7/10 (62.5\%) & 12/3/4 (71.1\%) \\
 & skill [95\% CI] & $\mathbf{+0.20}$ [$+0.06$, $+0.37$] & $\mathbf{+0.31}$ [$+0.10$, $+0.55$] & $\mathbf{+0.52}$ [$+0.21$, $+0.87$] & $\mathbf{+0.68}$ [$+0.17$, $+1.25$] \\
\midrule
\multirow{2}{*}{Moirai} & W/T/L (win rate) & 36/43/21 (57.5\%) & 36/9/21 (61.4\%) & 21/6/9 (66.7\%) & 11/3/5 (65.8\%) \\
 & skill [95\% CI] & $\mathbf{+0.23}$ [$+0.07$, $+0.40$] & $\mathbf{+0.35}$ [$+0.11$, $+0.60$] & $\mathbf{+0.42}$ [$+0.18$, $+0.74$] & $\mathbf{+0.28}$ [$+0.07$, $+0.51$] \\
\bottomrule
\end{tabular}
\end{center}
\end{table}

Since Chronos-2 tops the published fev-bench leaderboard at the time of writing and the deployed policy modifies only a small share of its forecasts, the selector inherits the backbone's dominance
over every published model essentially unchanged; the full pairwise
comparison is in Appendix~\ref{app:full-results}.

\section{Analysis}
\label{sec:analysis}

Chronos-2 is our primary backbone, and Moirai serves to check that the
main result carries over. Except for the depth analysis, which covers
both backbones, the analyses in this section use Chronos-2.

\paragraph{The signal strengthens with horizon depth.}
For Chronos-2, both metrics in Table~\ref{tab:main} increase monotonically as
evaluation is restricted to tasks with deeper horizons: win rate rises
from 54.5\% on the full benchmark to 71.1\% on tasks with $H\ge9$, and
skill rises from $+0.20\%$ to $+0.68\%$; the same ordering holds for
skill for the alternative configuration in
Appendix~\ref{app:full-results}. For Moirai, both metrics also rise from
the full benchmark to $H\ge4$ ($57.5\%$ to $66.7\%$, and $+0.23\%$ to
$+0.42\%$), but not further to $H\ge9$ ($65.8\%$, $+0.28\%$). This matches the design intuition
behind adaptive decoding: deeper horizons contain more decision points,
so a per-window choice of rollout schedule has more opportunities to
improve on the default one-shot strategy, while single-patch horizons admit no
decision at all. Fixed alternatives point the opposite way: none of the global fixed
schedules we evaluate improves on one-shot decoding on any subset, the
two finest are significantly worse, and their deficits grow with depth
(Appendix~\ref{app:full-results}), so these gains are not available to
any single schedule. Nor are they available to a fixed schedule chosen
per task or per horizon on the validation windows: the per-task choice
does not improve on one-shot decoding ($-0.25\%$ [$-0.60$, $+0.04$] on
$H\ge2$), and the per-horizon choice reduces to one-shot decoding
(Appendix~\ref{app:full-results}).

\paragraph{The statistical support comes from deeper horizons.}
The skill confidence intervals in Table~\ref{tab:main} exclude zero on
every subset for both backbones, including the full benchmark ($+0.20\%$,
[$+0.06$, $+0.37$] for Chronos-2 and $+0.23\%$, [$+0.07$, $+0.40$] for
Moirai). On disjoint depth bins (Appendix~\ref{app:full-results}), the
Chronos-2 policy is not significant on the 30 tasks with $H=2$--$3$
($+0.05\%$ [$-0.16$, $+0.39$]), but it is on the 17 tasks with
$H=4$--$8$ ($+0.34\%$ [$+0.04$, $+0.73$]) and on the 19 tasks with
$H\ge9$ ($+0.68\%$ [$+0.17$, $+1.25$]). The benchmark-wide significance
thus comes from the 36 tasks with $H\ge4$. With Moirai, the same bins
give $+0.25\%$ [$-0.16$, $+0.69$], $+0.59\%$ [$+0.14$, $+1.16$], and
$+0.28\%$ [$+0.07$, $+0.51$], so its significance also comes from
$H\ge4$. For Chronos-2, the deepest bin is significantly above the
shallowest ($+0.63$ [$+0.02$, $+1.25$]). For the alternative
configuration in Appendix~\ref{app:full-results}, the interval excludes
zero on the two deepest subsets, consistent with the same depth pattern.

\paragraph{Headroom, coverage, and the size of the effect.}
The full-benchmark skill of $+0.20\%$ is small for two reasons. First,
it is diluted: the 34 tasks with $H{=}1$ admit no decision and are
identical to one-shot decoding by construction, and on the 66 tasks with
$H\ge2$ the skill is $+0.31\%$. Second, the deployed policy is
conservative: it keeps the one-shot forecast unless four of the five
scorers agree, and on a median task it deviates on only 5 to 12\% of its
forecasts, where a forecast is one series in one official window
(Appendix~\ref{app:full-results}).
The per-task results in Appendix~\ref{app:full-results} provide an oracle reference for our selector.
The DP oracle, which selects with the ground
truth, attains a skill of $+3.25\%$ on the full benchmark and
$+9.13\%$ on $H\ge9$, so the deployed selector recovers roughly $6$--$8\%$ of that headroom. 

The oracle bound is structural rather than empirical: the one-shot rollout is always among
the final-row candidates, so ground-truth selection can never be worse
than one-shot decoding, which is why the oracle columns never lose in
Table~\ref{tab:pertask}. The oracle is also optimistic: it selects per
window with access to the realized future, so it exploits noise that
no causal predictor can anticipate, and its headroom is an upper bound
rather than a reachable target. The gap between the two is where
future gains lie. Since the horizon is known before inference, adaptive
decoding can also be applied selectively, to the long horizons where its
gain is largest.

\paragraph{Inference cost.}
We measured both decoding methods on one RTX 3090 on the 66 tasks with
$H\ge2$, using the first, middle, and last official window of each task,
identical inputs and batching, an fp32 backbone, and fp16 scorers; each
value is the median of three timed runs after two warm-up runs, excluding
model loading (Table~\ref{tab:cost}). Each DP candidate is currently a
separate backbone call, so the slowdown tracks the number of candidates,
$H(H{+}1)/2$: backbone calls take a median 97.6\% of the adaptive time
and scorer voting 1.7\%, while peak allocated GPU memory is almost
unchanged (median 0.86 vs 0.89 GiB). Candidates within a DP row are
independent and could be batched into one call, which would reduce the
number of sequential calls to $H$. The gain
thus comes at a cost that grows quadratically with the number of horizon
patches, and the method suits settings where accuracy matters more than
latency.

\begin{table}[htbp]
\caption{Accuracy and inference latency of the deployed policy with
Chronos-2 on one RTX 3090. Latency is the median over tasks of each task's
mean latency over its first, middle, and last official windows; slowdown
and $H(H{+}1)/2$ are geometric means over the tasks of each subset.}
\label{tab:cost}
\begin{center}
\small
\begin{tabular}{lccccc}
\toprule
 & skill & one-shot & adaptive & slowdown & $H(H{+}1)/2$ \\
\midrule
$H\ge2$ (66 tasks) & $+0.31\%$ & 0.25 s & 5.08 s & 13.7$\times$ & 12.6 \\
$H\ge4$ (36 tasks) & $+0.52\%$ & 0.31 s & 10.92 s & 42.3$\times$ & 39.8 \\
$H\ge9$ (19 tasks) & $+0.68\%$ & 0.25 s & 18.14 s & 103.6$\times$ & 96.6 \\
\bottomrule
\end{tabular}
\end{center}
\end{table}

\paragraph{Where the gain arises.}
Beyond depth, the gain of the Chronos-2 policy is concentrated in energy
tasks: on the 21 energy tasks with $H\ge2$ the skill is $+0.97\%$
[$+0.44$, $+1.58$] (12 wins, 6 ties, 3 losses), while the other domains
show small or no gains. Depth alone does not explain this, since the
mobility tasks are as deep (4 of 7 with $H\ge9$) yet show no gain. Across
tertiles of the one-shot MASE, the skill decreases from the easiest to
the hardest tertile ($+0.41\%$, $+0.26\%$, $+0.25\%$), but the hardest
tertile also holds fewer deep tasks, so this split does not separate
difficulty from depth. Per-domain results are in
Appendix~\ref{app:full-results}.

\section{Conclusion}

We formulated the decoding rollout schedule of frozen time-series
foundation models as a per-window inference-time decision problem, 
proposed a DP-based algorithm which makes the
exponential schedule space tractable with a prefix-retaining dynamic
program, and deployed the learned scorer through an agreement gate over
independently trained replicas that falls back to one-shot decoding
whenever the replicas disagree. The scorer is trained on historical windows
of the evaluated tasks, so the method adapts decoding to tasks whose history
is available rather than transferring to unseen tasks. On fev-bench with Chronos-2, the deployed
policy improves point-forecast skill over one-shot decoding with confidence
intervals that exclude zero on every nested subset of the benchmark, and both
the win rate and the skill gain grow as the forecast horizon deepens.
With Moirai-1.1-R-small, the confidence intervals also exclude zero on
every nested subset.
To the best of our knowledge, no fixed alternative offers this: none of the constant-chunk
schedules we evaluate improves on one-shot decoding, and the two
finest, including full autoregression, are significantly worse.

\makeatletter
\if@accepted
\subsubsection*{Acknowledgments}
This work was supported by Institute of Information \& Communications
Technology Planning \& Evaluation (IITP) grant funded by the Korea
government (MSIT) [No. \mbox{RS-2026-25524173}, Ultra-LongTerm Hierarchical Memory
and Reasoning Architecture for Next-Generation Omnimodal Agents, 40\%;
No. \mbox{RS-2026-25522672}, Development of Unified Reasoning Technology Mimicking
Human Cognition for Hierarchical Understanding and Unbounded Problem
Solving, 10\%; No. \mbox{RS-2020-II200882}, (SW STAR LAB) Development of
deployable learning intelligence via self-sustainable and trustworthy
machine learning, 10\%; and No. \mbox{RS-2021-II211343}, Artificial
Intelligence Graduate School Program (Seoul National University), 10\%],
and National Research Foundation of Korea (NRF)
grant funded by the Korea government (MSIT) (No. \mbox{RS-2024-00354036}, 30\%).
Hyun Oh Song is the corresponding author.
\fi
\makeatother

\bibliographystyle{tmlr}
\bibliography{references}

\appendix

\section{Training details}
\label{app:training}

\paragraph{Training data.}
For each task we collect candidate rollouts on historical windows that
strictly precede the official evaluation windows, running the DP of
Section~\ref{sec:dp} with randomized roll-in so that every prefix length is
visited once per window. Each visited DP row becomes one training
example, consisting of the $K$ candidate feature prefixes at that row
and, for each candidate, its prefix error measured against the known
ground truth. Training rows are thus collected along a randomized roll-in, whereas
at inference the scorer scores rows along its own trajectory.
Validation policies are likewise scored on the cached rows along the
fixed roll-in, a row-local proxy of the deployed closed-loop walk. Windows are capped at 10{,}000 per task; the 66 tasks with
$H\ge2$ contribute about 1.7M examples in total. A chronological split of
the historical windows provides the validation set; the official evaluation
windows are never touched during training or selection.

\paragraph{Objective and optimization.}
Within each DP row we form a target distribution by applying a softmax to
the negatives of the candidates' prefix errors, normalized per task and
scaled by a label temperature, and minimize the cross-entropy between the
scorer's softmax output and this target. Rows are batched by prefix length.
All scorers use the architecture of Section~\ref{sec:scorer} (width 256,
two encoder blocks per stage, four heads), AdamW with weight decay
$10^{-3}$, batch size 256, and 50 epochs; the checkpoint with the best validation improvement over one-shot (averaged per task) is selected.
Table~\ref{tab:configs} lists the hyperparameters of the deployed
selector and of the alternative configuration reported in
Appendix~\ref{app:full-results}; for each, five scorers are trained
that differ only in the random seed. The deployed configuration is the
one whose \emph{deployed policy} (the five-scorer ensemble with the
pre-fixed agreement threshold) performs best on the validation split:
65.2\% validation win rate and $+0.32\%$ validation skill, versus
63.6\% and $+0.22\%$ for the alternative.

\begin{table}[h]
\caption{Hyperparameters of the two selector configurations. All other
settings are shared; the deployed selector was chosen on validation
data.}
\label{tab:configs}
\begin{center}
\small
\begin{tabular}{lccc}
\toprule
 & label temperature & learning rate & dropout \\
\midrule
\textsc{Selector} (deployed) & 0.03 & $10^{-4}$ & 0.1 \\
\textsc{Selector} (alternative) & 0.1 & $3\times10^{-4}$ & 0.3 \\
\bottomrule
\end{tabular}
\end{center}
\end{table}

\paragraph{Second backbone.}
For Moirai-1.1-R-small, the training data are drawn from historical
windows that strictly precede the official evaluation windows, as for
Chronos-2, with at most 2{,}000 examples per task (122{,}884 in total).
Moirai runs with a context of 2{,}048 steps, without covariates, with a
fixed patch size of 16, and with the median of 5{,}000 samples as point
forecast, both for the training data and for evaluation. The scorer
architecture and optimization are those above, with Moirai's hidden width
of 384. The hyperparameters were selected on the validation split only.
We trained 72 configurations with one seed (label temperature in
$\{0.01, 0.03, 0.1, 0.3\}$, learning rate in $\{3\times10^{-5}, 10^{-4},
3\times10^{-4}\}$, dropout in $\{0.1, 0.3\}$, and the weight of an
additional loss term in $\{0, 0.3, 1\}$) and ranked them by the validation
improvement over one-shot. The four best configurations and the deployed
Chronos-2 configuration were trained with five seeds, and the
configuration whose deployed policy had the best validation skill was
chosen: label temperature 0.01, learning rate $10^{-4}$, dropout 0.1, and
weight 1 (validation skill $+0.60\%$). The additional loss term is the
expected logarithm, under the scorer's softmax output, of each
candidate's prefix error divided by that of the one-shot candidate, with
the ratios clipped to $[10^{-2},10^{2}]$ as in the skill.

\paragraph{Training on the policy's own states.}
Training rows follow a randomized roll-in, whereas the deployed policy
scores rows along its own trajectory (rows 1 and 2 do not depend on the
roll-in, so only tasks with $H\ge3$ are affected). To test whether training
on the policy's own states helps, we ran one round of dataset aggregation
\citep[DAgger;][]{ross2011reduction}: the deployed policy decoded every
training window of the 38 tasks with $H\ge3$, and the rows it visited were
added to the original training rows in a 1:1 ratio (for $H{=}2$ the
original rows are repeated, keeping the task weights). We retrained with
the deployed configuration, seeds, and checkpoint rule, and evaluated once
on the official windows (Table~\ref{tab:dagger}). Training on the policy's
own states raises the win rate on $H\ge4$ (69.4\% vs 62.5\%) and $H\ge9$
(76.3\% vs 71.1\%), but not the skill: no paired difference is
significant, and the point estimates are slightly lower. One round of
aggregation thus does not improve the method, although it does not rule
out gains from further rounds, and the mismatch between the training and
deployment states remains.

\begin{table}[H]
\caption{One round of dataset aggregation (DAgger) against the deployed
selector with Chronos-2, under the protocol of Table~\ref{tab:main}: skill
[95\% CI] (W/T/L), and the paired skill difference.}
\label{tab:dagger}
\begin{center}
\small
\setlength{\tabcolsep}{4pt}
\resizebox{\textwidth}{!}{%
\begin{tabular}{lccc}
\toprule
 & \textbf{$H\ge2$} (66 tasks) & \textbf{$H\ge4$} (36 tasks) & \textbf{$H\ge9$} (19 tasks) \\
\midrule
deployed & $\mathbf{+0.31}$ [$+0.10$, $+0.55$] (29/17/20) & $\mathbf{+0.52}$ [$+0.21$, $+0.87$] (19/7/10) & $\mathbf{+0.68}$ [$+0.17$, $+1.25$] (12/3/4) \\
DAgger, one round & $+0.14$ [$-0.13$, $+0.44$] (28/17/21) & $+0.29$ [$-0.13$, $+0.74$] (22/6/8) & $+0.56$ [$-0.16$, $+1.33$] (14/1/4) \\
DAgger $-$ deployed (paired) & $-0.17$ [$-0.40$, $+0.04$] & $-0.23$ [$-0.64$, $+0.13$] & $-0.12$ [$-0.51$, $+0.35$] \\
\bottomrule
\end{tabular}}
\end{center}
\end{table}

\section{Agreement gate: rationale and details}
\label{app:gate}

In our validation diagnostics, the margin between the top-ranked
candidate and the one-shot candidate is a poor guide: it barely
distinguishes helpful deviations from harmful ones. Agreement across
independently trained replicas does, so the gate votes rather than
thresholding a margin. Voting also bears on a structural property of the
candidate set. Because the one-shot rollout
is always a candidate, ground-truth selection is never worse than
one-shot decoding, while a scorer that mis-ranks a candidate can be.
A learned scorer forfeits this guarantee. The agreement gate does not
restore it; it makes deviations more conservative, and in our experiments
empirically more reliable (Table~\ref{tab:ablation}). It does so by
concentrating deviations where the ensemble is confident.

The deployed threshold, four of five, tolerates one dissenting
replica. It was fixed before any official evaluation rather than tuned
on validation, because the validation metric is only a row-local proxy
of the deployed closed-loop walk (Appendix~\ref{app:training}) and
cannot resolve adjacent thresholds. The ablation in
Appendix~\ref{app:full-results} varies the number of scorers $n$ and the
agreement threshold $m$ after the fact and finds the mechanism
insensitive to the exact choice when one dissent is allowed
($m=n-1$, $n\ge3$). At inference, the backbone's candidate
generation dominates the cost and is computed once, shared by all
scorers. The ensemble itself adds four extra forward passes of the
lightweight scorer per decision (3.4M parameters, against 120M for the
backbone), negligible relative to a single backbone call.

\section{Full results}
\label{app:full-results}

\subsection{Baselines}
\label{app:baselines}

\paragraph{Fixed-schedule baselines.}
Table~\ref{tab:fixed} evaluates deployable fixed rollout schedules
under the protocol of Table~\ref{tab:main}: fully autoregressive
decoding (chunks of one patch) and constant chunks of two and four
patches, rolled with the same pipeline as the selector. A task whose
horizon does not exceed the chunk coincides with one-shot decoding and
counts as a tie. None of these schedules improves on one-shot decoding
on any subset, and the deficit widens as the chunks become finer.
Chunks of one and two are significantly worse: every skill interval
lies below zero, and full autoregression loses all 19 tasks with
$H\ge9$ (skill $-10.7\%$). The best fixed alternative, a constant
chunk of four, is statistically indistinguishable from one-shot
($-0.16\%$ on the full benchmark, $-0.64\%$ on $H\ge9$, intervals
spanning zero), while still losing more tasks than it wins. At the
same time these schedules do win individual tasks (chunk four wins 9
of 19 on $H\ge9$), confirming that no single schedule dominates.

\begin{table}[H]
\caption{Deployable fixed rollout schedules under the protocol of
Table~\ref{tab:main}. A constant chunk of $c$ patches decodes the
horizon in consecutive $c$-patch calls, with a final shorter call
when $c$ does not divide $H$ ($c{=}1$ is fully autoregressive); tasks
with $H \le c$ coincide with one-shot decoding and tie.}
\label{tab:fixed}
\begin{center}
\small
\setlength{\tabcolsep}{4pt}
\resizebox{\textwidth}{!}{%
\begin{tabular}{lcccc}
\toprule
 & \multicolumn{2}{c}{\textbf{full} (100 tasks)} & \multicolumn{2}{c}{\textbf{$H\ge9$} (19 tasks)} \\
 & W/T/L (win rate) & skill [95\% CI] & W/T/L (win rate) & skill [95\% CI] \\
\midrule
fully autoregressive & 11/34/55 (28.0\%) & $-3.11$ [$-4.43$, $-1.95$] & 0/0/19 (0.0\%) & $-10.73$ [$-16.04$, $-6.28$] \\
constant chunk 2 & 11/62/27 (42.0\%) & $-1.18$ [$-1.89$, $-0.57$] & 4/0/15 (21.1\%) & $-4.76$ [$-7.71$, $-2.18$] \\
constant chunk 4 & 12/72/16 (48.0\%) & $-0.16$ [$-0.49$, $+0.13$] & 9/0/10 (47.4\%) & $-0.64$ [$-2.31$, $+0.81$] \\
\bottomrule
\end{tabular}}
\end{center}
\end{table}

\paragraph{Schedules chosen per task or per horizon.}
To separate per-window adaptation from adaptation across tasks, we also
choose a fixed schedule from one-shot decoding and constant chunks of 1,
2, and 4 patches on the validation windows, either per task (lowest total
validation MASE) or per horizon $H$ (lowest mean log MASE ratio over the
tasks with that $H$), and evaluate it once on the official windows
(Table~\ref{tab:fixed-task}). 
On validation, one-shot decoding is the best
schedule at every horizon, so the per-horizon choice reduces to one-shot
decoding. The per-task choice picks a chunked schedule for 17 of the 66
tasks, but these choices do not carry over to the official windows (7
wins, 49 ties, 10 losses on $H\ge2$). A fixed schedule chosen per task or
per horizon thus does not reproduce the gain of per-window selection.

\begin{table}[H]
\caption{Fixed schedules chosen on the validation windows per task or per
horizon, under the protocol of Table~\ref{tab:main}: skill [95\% CI], and
the paired skill difference of the deployed selector over the per-task
schedule.}
\label{tab:fixed-task}
\begin{center}
\small
\setlength{\tabcolsep}{4pt}
\begin{tabular}{lccc}
\toprule
 & \textbf{$H\ge2$} (66 tasks) & \textbf{$H\ge4$} (36 tasks) & \textbf{$H\ge9$} (19 tasks) \\
\midrule
deployed (per window) & $\mathbf{+0.31}$ [$+0.10$, $+0.55$] & $\mathbf{+0.52}$ [$+0.21$, $+0.87$] & $\mathbf{+0.68}$ [$+0.17$, $+1.25$] \\
per-task schedule & $-0.25$ [$-0.60$, $+0.04$] & $-0.41$ [$-1.05$, $+0.11$] & $-0.43$ [$-1.51$, $+0.37$] \\
per-horizon schedule & $0.00$ & $0.00$ & $0.00$ \\
deployed $-$ per-task (paired) & $\mathbf{+0.56}$ [$+0.19$, $+1.00$] & $\mathbf{+0.93}$ [$+0.31$, $+1.67$] & $\mathbf{+1.12}$ [$+0.15$, $+2.37$] \\
\bottomrule
\end{tabular}
\end{center}
\end{table}

\paragraph{Ensemble of candidate forecasts.}
A simpler way to use several schedules is to combine their forecasts
without a scorer. We combine the forecasts of one-shot decoding and the
constant chunks of 1, 2, and 4 patches of Table~\ref{tab:fixed}
pointwise, by their median or their mean, and evaluate the result on the
official windows (Table~\ref{tab:cand-ens}). Both ensembles are
significantly worse than one-shot decoding on every subset, so combining
the schedules without a scorer does not reproduce the gain of per-window
selection.

\begin{table}[H]
\caption{Pointwise ensembles of the forecasts of one-shot decoding and the
constant chunks of 1, 2, and 4 patches, without a scorer, under the
protocol of Table~\ref{tab:main}: skill [95\% CI]. Skill estimates whose
95\% confidence interval excludes zero are in bold.}
\label{tab:cand-ens}
\begin{center}
\small
\setlength{\tabcolsep}{4pt}
\begin{tabular}{lcccc}
\toprule
 & \textbf{full} & \textbf{$H\ge2$} & \textbf{$H\ge4$} & \textbf{$H\ge9$} \\
 & (100 tasks) & (66 tasks) & (36 tasks) & (19 tasks) \\
\midrule
deployed & $\mathbf{+0.20}$ [$+0.06$, $+0.37$] & $\mathbf{+0.31}$ [$+0.10$, $+0.55$] & $\mathbf{+0.52}$ [$+0.21$, $+0.87$] & $\mathbf{+0.68}$ [$+0.17$, $+1.25$] \\
pointwise median & $\mathbf{-0.49}$ [$-0.90$, $-0.14$] & $\mathbf{-0.74}$ [$-1.36$, $-0.23$] & $\mathbf{-1.33}$ [$-2.41$, $-0.40$] & $\mathbf{-2.01}$ [$-3.87$, $-0.43$] \\
pointwise mean & $\mathbf{-0.54}$ [$-0.94$, $-0.20$] & $\mathbf{-0.83}$ [$-1.42$, $-0.33$] & $\mathbf{-1.34}$ [$-2.35$, $-0.45$] & $\mathbf{-2.18}$ [$-3.94$, $-0.67$] \\
\bottomrule
\end{tabular}
\end{center}
\end{table}

\paragraph{Simpler scorers.}
To measure what the candidate comparison of the scorer contributes, we
train two simpler scorers with the same data, training configuration,
checkpoint rule, five seeds, and agreement gate as the deployed scorer
(3.4M parameters): (A) a pointwise MLP that scores each candidate feature
sequence on its own, with no attention and no comparison between
candidates (0.94M), and (B) a context-only scorer that reads only the
feature sequence of the one-shot candidate, which encodes the input
context, through Stage~1 of our scorer and an MLP head over the candidate
index (1.86M). Neither improves significantly on one-shot decoding over
the 66 tasks, and the deployed scorer is significantly better than both
on $H\ge2$ and $H\ge4$ (Table~\ref{tab:simple}). On $H\ge9$ the
context-only scorer keeps a small gain by deviating rarely (one loss in 19
tasks), so the context alone identifies a few safe deviations, but most of
the gain requires comparing the candidates.

\begin{table}[H]
\caption{Simpler scorers under the protocol of Table~\ref{tab:main}:
skill [95\% CI], and the paired skill difference of the deployed scorer
over each.}
\label{tab:simple}
\begin{center}
\small
\setlength{\tabcolsep}{4pt}
\begin{tabular}{lccc}
\toprule
 & \textbf{$H\ge2$} (66 tasks) & \textbf{$H\ge4$} (36 tasks) & \textbf{$H\ge9$} (19 tasks) \\
\midrule
deployed & $\mathbf{+0.31}$ [$+0.10$, $+0.55$] & $\mathbf{+0.52}$ [$+0.21$, $+0.87$] & $\mathbf{+0.68}$ [$+0.17$, $+1.25$] \\
(A) pointwise MLP & $+0.04$ [$-0.11$, $+0.21$] & $+0.16$ [$-0.08$, $+0.45$] & $+0.26$ [$-0.04$, $+0.74$] \\
(B) context-only & $+0.04$ [$-0.05$, $+0.17$] & $+0.13$ [$-0.00$, $+0.34$] & $\mathbf{+0.21}$ [$+0.01$, $+0.58$] \\
deployed $-$ (A) (paired) & $\mathbf{+0.27}$ [$+0.08$, $+0.49$] & $\mathbf{+0.36}$ [$+0.10$, $+0.67$] & $+0.42$ [$-0.04$, $+0.97$] \\
deployed $-$ (B) (paired) & $\mathbf{+0.27}$ [$+0.07$, $+0.50$] & $\mathbf{+0.39}$ [$+0.11$, $+0.71$] & $\mathbf{+0.47}$ [$+0.02$, $+0.98$] \\
\bottomrule
\end{tabular}
\end{center}
\end{table}

\subsection{Configuration, ensemble size, and transfer}
\label{app:robustness}

\paragraph{Alternative configuration.}
Table~\ref{tab:alt} reports the alternative selector configuration of
Table~\ref{tab:configs} under the protocol of Table~\ref{tab:main}. Its
skill also increases monotonically with horizon depth, and its skill
intervals exclude zero on the two deepest subsets, so the depth pattern
of Section~\ref{sec:analysis} is not specific to the deployed
hyperparameter setting.

\begin{table}[H]
\caption{The alternative selector configuration under the protocol of
Table~\ref{tab:main}. Skill estimates whose 95\% confidence interval
excludes zero are in bold.}
\label{tab:alt}
\begin{center}
\small
\setlength{\tabcolsep}{4pt}
\begin{tabular}{lcccc}
\toprule
 & \textbf{full} & \textbf{$H\ge2$} & \textbf{$H\ge4$} & \textbf{$H\ge9$} \\
 & (100 tasks) & (66 tasks) & (36 tasks) & (19 tasks) \\
\midrule
W/T/L (win rate) & 26/57/17 (54.5\%) & 26/23/17 (56.8\%) & 18/10/8 (63.9\%) & 10/4/5 (63.2\%) \\
skill [95\% CI] & $+0.12$ [$-0.00$, $+0.26$] & $+0.18$ [$-0.01$, $+0.40$] & $\mathbf{+0.33}$ [$+0.03$, $+0.68$] & $\mathbf{+0.58}$ [$+0.12$, $+1.16$] \\
\bottomrule
\end{tabular}
\end{center}
\end{table}

\paragraph{Ensemble size and agreement threshold.}
Table~\ref{tab:ablation} varies the number of scorers $n$ and the
agreement threshold $m$ of the deployed configuration, with members
added in a fixed seed order; $n{=}1$ is a single scorer with no gate.
Three patterns emerge. First, a single ungated scorer is high-variance
rather than weak. Across the five seeds, the full-benchmark win rate
ranges from 48.5\% to 54.0\%, and no skill interval excludes zero. On
$H\ge9$, four of the five seeds reach a larger point estimate than the
deployed policy, but with intervals two to three times wider, one
spanning zero from $-1.04$. The ensemble and gate therefore trade the
best seed's magnitude for reliability: intervals that exclude zero,
and a policy that does not depend on which seed was drawn.

Second, requiring unanimity filters the value away. Along the unanimity
cells, full-benchmark skill decreases monotonically with ensemble size
($+0.12$, $+0.10$, $+0.06$, $+0.03$ for $n{=}2,3,4,5$). Third, gates
that allow one dissent ($m{=}n{-}1$, $n\ge3$) are uniformly strong. All
three cells attain skill intervals that exclude zero on the full
benchmark, and several match or exceed the deployed cell's point
estimates. For example, $n{=}4$, $m{=}3$ reaches an 84\% win rate on
$H\ge9$, yet a paired comparison with the deployed policy shows no
measurable difference ($-0.01\%$ paired skill, 95\% CI $[-0.07,
+0.06]$). The mechanism is thus not sensitive to the exact $(n, m)$
among these gates, and the pre-fixed deployed choice is one of them.

\begin{table}[H]
\caption{Ensemble-size ablation under the protocol of
Table~\ref{tab:main}: $n$ scorers with agreement threshold $m$, members
added in a fixed seed order. W/T/L (win rate) and skill [95\% CI] on
the full benchmark and the $H\ge9$ subset. Skill estimates whose 95\% confidence interval excludes zero are in bold.}
\label{tab:ablation}
\begin{center}
\small
\setlength{\tabcolsep}{4pt}
\resizebox{\textwidth}{!}{%
\begin{tabular}{lcccc}
\toprule
 & \multicolumn{2}{c}{\textbf{full} (100 tasks)} & \multicolumn{2}{c}{\textbf{$H\ge9$} (19 tasks)} \\
 & W/T/L (win rate) & skill [95\% CI] & W/T/L (win rate) & skill [95\% CI] \\
\midrule
$n{=}1$, seed 42 & 29/41/30 (49.5\%) & $+0.17$ [$-0.11$, $+0.49$] & 12/1/6 (65.8\%) & $+1.19$ [$-0.05$, $+2.57$] \\
$n{=}1$, seed 43 & 29/39/32 (48.5\%) & $+0.20$ [$-0.12$, $+0.54$] & 13/1/5 (71.1\%) & $\mathbf{+1.60}$ [$+0.24$, $+3.01$] \\
$n{=}1$, seed 44 & 31/46/23 (54.0\%) & $+0.26$ [$-0.01$, $+0.56$] & 11/1/7 (60.5\%) & $+1.24$ [$-0.05$, $+2.49$] \\
$n{=}1$, seed 45 & 28/41/31 (48.5\%) & $+0.03$ [$-0.26$, $+0.31$] & 10/0/9 (52.6\%) & $+0.32$ [$-1.04$, $+1.62$] \\
$n{=}1$, seed 46 & 30/40/30 (50.0\%) & $+0.16$ [$-0.16$, $+0.51$] & 12/0/7 (63.2\%) & $+1.17$ [$-0.13$, $+2.62$] \\
\midrule
$n{=}2$, $m{=}2$ & 28/46/26 (51.0\%) & $+0.12$ [$-0.03$, $+0.30$] & 14/1/4 (76.3\%) & $\mathbf{+0.75}$ [$+0.12$, $+1.50$] \\
$n{=}3$, $m{=}2$ & 32/45/23 (54.5\%) & $\mathbf{+0.33}$ [$+0.09$, $+0.60$] & 13/1/5 (71.1\%) & $\mathbf{+1.42}$ [$+0.48$, $+2.50$] \\
$n{=}3$, $m{=}3$ & 28/54/18 (55.0\%) & $\mathbf{+0.10}$ [$+0.00$, $+0.20$] & 14/3/2 (81.6\%) & $\mathbf{+0.54}$ [$+0.16$, $+0.92$] \\
$n{=}4$, $m{=}3$ & 33/49/18 (57.5\%) & $\mathbf{+0.20}$ [$+0.04$, $+0.37$] & 15/2/2 (84.2\%) & $\mathbf{+0.69}$ [$+0.08$, $+1.34$] \\
$n{=}4$, $m{=}4$ & 28/55/17 (55.5\%) & $+0.06$ [$-0.02$, $+0.14$] & 13/3/3 (76.3\%) & $\mathbf{+0.36}$ [$+0.01$, $+0.72$] \\
$n{=}5$, $m{=}4$ (deployed) & 29/51/20 (54.5\%) & $\mathbf{+0.20}$ [$+0.06$, $+0.37$] & 12/3/4 (71.1\%) & $\mathbf{+0.68}$ [$+0.17$, $+1.25$] \\
$n{=}5$, $m{=}5$ & 24/59/17 (53.5\%) & $+0.03$ [$-0.04$, $+0.09$] & 10/5/4 (65.8\%) & $+0.17$ [$-0.09$, $+0.42$] \\
\bottomrule
\end{tabular}}
\end{center}
\end{table}

\paragraph{Seed subsets.}
Because members are added in one fixed seed order, each cell of
Table~\ref{tab:ablation} with $n=2$ to $4$ uses a single subset of the
five scorers. Table~\ref{tab:ablation-seeds} repeats these cells with
the four other cyclic orders, starting at seeds 43, 44, 45, and 46,
which gives five subsets per cell; for $n=4$ these are all five
subsets. Skill is positive for every subset, and no subset is
significantly worse than one-shot decoding on the full benchmark,
$H\ge2$, or $H\ge9$. The gate matters more than the choice of seeds:
with one dissent allowed ($m=n-1$), the full-benchmark interval
excludes zero for 9 of the 10 subsets, against 5 of the 15 subsets with
unanimous gates.

\begin{table}[htbp]
\caption{Ensemble-size ablation over seed subsets, under the protocol of
Table~\ref{tab:main}. Each cell is evaluated on five subsets of the five
scorers, one per cyclic seed order: skill median (minimum to maximum),
and the number of subsets whose 95\% confidence interval lies above zero.}
\label{tab:ablation-seeds}
\begin{center}
\small
\setlength{\tabcolsep}{4pt}
\begin{tabular}{lcccc}
\toprule
 & \multicolumn{2}{c}{\textbf{full} (100 tasks)} & \multicolumn{2}{c}{\textbf{$H\ge9$} (19 tasks)} \\
 & skill, median (min to max) & CI $>0$ & skill, median (min to max) & CI $>0$ \\
\midrule
$n{=}2$, $m{=}2$ & $+0.14$ ($+0.12$ to $+0.15$) & 2/5 & $+0.61$ ($+0.42$ to $+0.75$) & 2/5 \\
$n{=}3$, $m{=}2$ & $+0.28$ ($+0.16$ to $+0.33$) & 4/5 & $+1.11$ ($+0.94$ to $+1.42$) & 4/5 \\
$n{=}3$, $m{=}3$ & $+0.09$ ($+0.06$ to $+0.17$) & 2/5 & $+0.41$ ($+0.18$ to $+0.54$) & 3/5 \\
$n{=}4$, $m{=}3$ & $+0.21$ ($+0.18$ to $+0.25$) & 5/5 & $+0.69$ ($+0.61$ to $+1.05$) & 4/5 \\
$n{=}4$, $m{=}4$ & $+0.05$ ($+0.03$ to $+0.11$) & 1/5 & $+0.27$ ($+0.11$ to $+0.36$) & 2/5 \\
$n{=}5$, $m{=}4$ (deployed) & $+0.20$ & 1/1 & $+0.68$ & 1/1 \\
\bottomrule
\end{tabular}
\end{center}
\end{table}

\paragraph{Transfer to unseen tasks.}
The scorer is trained on historical windows of the evaluated tasks, and
transfer to unseen tasks is not part of our claim. To make this scope
concrete, we split the 66 tasks with $H\ge2$ into five folds. For each
fold, five scorers are trained on the other four folds (80\% of the
tasks) with the same configuration and agreement gate, with checkpoints
selected on their validation windows, and evaluated on the official
windows of the held-out fold (Table~\ref{tab:transfer}). On unseen tasks
the scorer does not improve significantly on one-shot decoding at any
depth, and it is significantly worse than the deployed scorer on $H\ge4$
and $H\ge9$, where the gains are concentrated. The gain thus depends on
the history of the evaluated tasks.

\begin{table}[H]
\caption{Scorers evaluated on tasks unseen in training (five folds of the
66 tasks with $H\ge2$), under the protocol of Table~\ref{tab:main}: skill
[95\% CI] (W/T/L), and the paired skill difference to the deployed
selector.}
\label{tab:transfer}
\begin{center}
\small
\setlength{\tabcolsep}{4pt}
\resizebox{\textwidth}{!}{%
\begin{tabular}{lccc}
\toprule
 & \textbf{$H\ge2$} (66 tasks) & \textbf{$H\ge4$} (36 tasks) & \textbf{$H\ge9$} (19 tasks) \\
\midrule
deployed & $\mathbf{+0.31}$ [$+0.10$, $+0.55$] (29/17/20) & $\mathbf{+0.52}$ [$+0.21$, $+0.87$] (19/7/10) & $\mathbf{+0.68}$ [$+0.17$, $+1.25$] (12/3/4) \\
unseen tasks & $+0.17$ [$-0.08$, $+0.48$] (17/20/29) & $+0.05$ [$-0.21$, $+0.33$] (10/10/16) & $-0.01$ [$-0.48$, $+0.48$] (6/4/9) \\
unseen $-$ deployed (paired) & $-0.14$ [$-0.41$, $+0.18$] & $\mathbf{-0.47}$ [$-0.81$, $-0.17$] & $\mathbf{-0.69}$ [$-1.29$, $-0.17$] \\
\bottomrule
\end{tabular}}
\end{center}
\end{table}

\subsection{Depth, domain, and selected schedules}
\label{app:where}

\paragraph{Disjoint depth bins.}
Table~\ref{tab:bins} reports the deployed policy on disjoint bins of the
66 tasks with $H\ge2$, under the protocol of Table~\ref{tab:main}; the
$H\ge9$ bin is the $H\ge9$ subset of Table~\ref{tab:main}. For
Chronos-2 it also gives the skill of ground-truth selection within the
same dynamic program (DP oracle) on each bin, and the share of it that
the deployed policy recovers. Section~\ref{sec:analysis} discusses the
pattern.

\begin{table}[H]
\caption{The deployed selector on disjoint depth bins of the 66 tasks
with $H\ge2$, under the protocol of Table~\ref{tab:main}. DP oracle is
the skill of ground-truth selection within the same dynamic program, and
recovery is the deployed skill divided by it (Chronos-2 only). Skill
estimates whose 95\% confidence interval excludes zero are in bold.}
\label{tab:bins}
\begin{center}
\small
\setlength{\tabcolsep}{4pt}
\begin{tabular}{llcccc}
\toprule
 & $H$ & W/T/L (win rate) & skill [95\% CI] & DP oracle & recovery \\
\midrule
\multirow{3}{*}{Chronos-2} & 2--3 (30 tasks) & 10/10/10 (50.0\%) & $+0.05$ [$-0.16$, $+0.39$] & $+2.29$ & 2.4\% \\
 & 4--8 (17 tasks) & 7/4/6 (52.9\%) & $\mathbf{+0.34}$ [$+0.04$, $+0.73$] & $+4.53$ & 7.5\% \\
 & $\ge9$ (19 tasks) & 12/3/4 (71.1\%) & $\mathbf{+0.68}$ [$+0.17$, $+1.25$] & $+9.13$ & 7.5\% \\
\midrule
\multirow{3}{*}{Moirai} & 2--3 (30 tasks) & 15/3/12 (55.0\%) & $+0.25$ [$-0.16$, $+0.69$] & --- & --- \\
 & 4--8 (17 tasks) & 10/3/4 (67.6\%) & $\mathbf{+0.59}$ [$+0.14$, $+1.16$] & --- & --- \\
 & $\ge9$ (19 tasks) & 11/3/5 (65.8\%) & $\mathbf{+0.28}$ [$+0.07$, $+0.51$] & --- & --- \\
\bottomrule
\end{tabular}
\end{center}
\end{table}

\paragraph{Gain by domain and difficulty.}
Table~\ref{tab:domain} breaks down the deployed policy with Chronos-2 on
the 66 tasks with $H\ge2$ by the fev-bench domain of each task (all econ
tasks have $H{=}1$), as assigned by \citet[Appendix~A]{shchur2025fev}, and
by tertiles of the one-shot MASE on the official windows;
Section~\ref{sec:analysis} discusses the pattern.

\begin{table}[H]
\caption{The deployed selector with Chronos-2 by fev-bench domain and by
tertiles of the one-shot MASE (22 tasks each), on the 66 tasks with
$H\ge2$: skill [95\% CI].}
\label{tab:domain}
\begin{center}
\small
\setlength{\tabcolsep}{4pt}
\begin{tabular}{lccc}
\toprule
 & tasks ($H\ge4$) & W/T/L & skill [95\% CI] \\
\midrule
energy & 21 (12) & 12/6/3 & $\mathbf{+0.97}$ [$+0.44$, $+1.58$] \\
cloud & 20 (15) & 7/8/5 & $+0.11$ [$-0.03$, $+0.28$] \\
retail & 9 (2) & 5/0/4 & $-0.04$ [$-0.18$, $+0.08$] \\
mobility & 7 (5) & 3/0/4 & $-0.07$ [$-0.33$, $+0.10$] \\
nature & 5 (2) & 1/3/1 & $-0.03$ [$-0.85$, $+0.75$] \\
healthcare & 4 (0) & 1/0/3 & $-0.32$ [$-0.69$, $+0.01$] \\
\midrule
easiest one-shot MASE tertile & 22 (15) & 7/6/9 & $\mathbf{+0.41}$ [$+0.05$, $+0.86$] \\
middle tertile & 22 (14) & 12/6/4 & $\mathbf{+0.26}$ [$+0.02$, $+0.58$] \\
hardest tertile & 22 (7) & 10/5/7 & $+0.25$ [$-0.16$, $+0.78$] \\
\bottomrule
\end{tabular}
\end{center}
\end{table}

\paragraph{Selected schedules.}
Table~\ref{tab:schedules} summarizes the per-row decisions of the
deployed policy with Chronos-2 on all official windows of the 66 tasks
with $H\ge2$, computed per task and then summarized over tasks. The
policy is selective: one-shot decoding remains the default, and on a
median task only 5 to 12\% of series-window pairs deviate. Schedules become more
fragmented with depth, yet full autoregression is never chosen for
$H\ge9$, consistent with the significantly worse finest fixed schedules
(Table~\ref{tab:fixed}). Among two-chunk schedules on longer horizons, a
shorter first chunk is about twice as frequent as a longer one, and
schedules vary across series and windows of the same task, supporting
input-dependent selection.

\begin{table}[H]
\caption{Rollout schedules selected by the deployed policy with Chronos-2
on the official windows. The policy selects a schedule for each series in
each official window; a deviating pair is a series-window pair whose final
schedule is not one-shot decoding, and a chunk is one decoded block of the
schedule.}
\label{tab:schedules}
\begin{center}
\small
\setlength{\tabcolsep}{4pt}
\begin{tabular}{lcccc}
\toprule
 & tasks & deviating pairs per task & chunks per & distinct schedules per \\
$H$ & (deviating) & median / max & deviating schedule & deviating task, median \\
\midrule
2--3 & 30 (20) & 5.1\% / 57.4\% & 2.0 & 1 \\
4--8 & 17 (13) & 12.4\% / 70.0\% & 2.6 & 5 \\
$\ge9$ & 19 (16) & 5.9\% / 35.0\% & 3.2 & 10.5 \\
\bottomrule
\end{tabular}
\end{center}
\end{table}

\subsection{Probabilistic metrics}
\label{app:prob}

Our claims concern point accuracy. As a check, Table~\ref{tab:prob}
reports the two probabilistic metrics of fev-bench, the scaled quantile
loss (SQL) and the weighted quantile loss (WQL), for the deployed policy
with Chronos-2. The quantiles of each chunk are taken from the backbone
call that produced it. The change is not significant on the full
benchmark or on $H\ge2$, and the probabilistic metrics are slightly
worse on some depth bins. We therefore make no claim for probabilistic
forecasts.

\begin{table}[H]
\caption{Probabilistic metrics of the deployed selector with Chronos-2
against one-shot decoding, under the protocol of Table~\ref{tab:main}
with SQL or WQL in place of MASE: W/T/L on SQL and skill [95\% CI].
The depth bins are those of Table~\ref{tab:bins}. Skill estimates whose
95\% confidence interval excludes zero are in bold.}
\label{tab:prob}
\begin{center}
\small
\setlength{\tabcolsep}{4pt}
\begin{tabular}{lcccc}
\toprule
 & tasks & W/T/L (SQL) & SQL skill [95\% CI] & WQL skill [95\% CI] \\
\midrule
full & 100 & 14/51/35 & $-0.10$ [$-0.24$, $+0.04$] & $-0.07$ [$-0.21$, $+0.07$] \\
$H\ge2$ & 66 & 14/17/35 & $-0.15$ [$-0.37$, $+0.06$] & $-0.11$ [$-0.32$, $+0.11$] \\
\midrule
$H=2$--$3$ & 30 & 3/10/17 & $-0.29$ [$-0.56$, $+0.00$] & $\mathbf{-0.16}$ [$-0.28$, $-0.03$] \\
$H=4$--$8$ & 17 & 2/4/11 & $\mathbf{-0.25}$ [$-0.47$, $-0.07$] & $\mathbf{-0.38}$ [$-0.80$, $-0.02$] \\
$H\ge9$ & 19 & 9/3/7 & $+0.16$ [$-0.43$, $+0.68$] & $+0.22$ [$-0.38$, $+0.82$] \\
\bottomrule
\end{tabular}
\end{center}
\end{table}

\subsection{Published models and per-task results}
\label{app:published}

\paragraph{Position relative to published models.}
Tables~\ref{tab:pairwise} and~\ref{tab:pairwise-skill} place the
deployed selector in the context of the published fev-bench
leaderboard, using the benchmark's official pairwise win-rate and
skill computations. The column models are
TiRex~\citep{auer2025tirex}, TimesFM-2.5, a later release of
TimesFM~\citep{das2024timesfm}, Toto-1.0~\citep{cohen2025toto},
TabPFN-TS~\citep{hoo2025tables}, the benchmark's statistical-ensemble
and seasonal-naive baselines~\citep{shchur2025fev},
Sundial~\citep{liu2025sundial}, Moirai-2.0~\citep{liu2025moirai2},
and Chronos-Bolt, a later, faster release in the
Chronos family~\citep{ansari2024chronos}. Because the backbone already dominates
every published model, the columns against third models inherit that
dominance essentially unchanged; the informative column is the
comparison against one-shot Chronos-2 itself.

\begin{table}[H]
\caption{Pairwise win rate (\%) between the row model and the column model
on all 100 tasks, computed with the official fev-bench recipe (MASE).
Columns against models other than one-shot Chronos-2 largely reflect the
frozen backbone shared by both rows. All cells here follow the
benchmark's cross-model recipe; the paired within-pipeline comparison
against one-shot decoding is in Table~\ref{tab:main}.}
\label{tab:pairwise}
\begin{center}
\small
\resizebox{\textwidth}{!}{%
\begin{tabular}{lccccccccccc}
\toprule
vs $\rightarrow$ & \textsc{Selector} & one-shot & TiRex & TimesFM-2.5 & Toto-1.0 & TabPFN-TS & Stat.\ Ens. & Sundial & Moirai-2.0 & Chronos-Bolt & Seas.\ Naive \\
\midrule
one-shot Chronos-2 & 45.0 & --- & 72.0 & 71.0 & 74.0 & 80.0 & 86.0 & 91.0 & 92.0 & 94.0 & 98.0 \\
\textsc{Selector}  & ---  & \textbf{55.0} & 72.0 & 71.0 & 74.0 & 81.0 & 86.0 & 92.0 & 91.0 & 94.0 & 98.0 \\
\bottomrule
\end{tabular}}
\end{center}
\end{table}

\begin{table}[H]
\caption{Pairwise skill (\%) of the row model over the column model on
all 100 tasks, computed with the official fev-bench recipe (MASE);
positive values favor the row model. As in Table~\ref{tab:pairwise},
columns against models other than one-shot Chronos-2 largely reflect
the shared frozen backbone.}
\label{tab:pairwise-skill}
\begin{center}
\small
\resizebox{\textwidth}{!}{%
\begin{tabular}{lccccccccccc}
\toprule
vs $\rightarrow$ & \textsc{Selector} & one-shot & TiRex & TimesFM-2.5 & Toto-1.0 & TabPFN-TS & Stat.\ Ens. & Sundial & Moirai-2.0 & Chronos-Bolt & Seas.\ Naive \\
\midrule
one-shot Chronos-2 & $-0.2$ & --- & $+7.8$ & $+7.6$ & $+10.1$ & $+8.1$ & $+22.8$ & $+14.3$ & $+11.4$ & $+12.2$ & $+35.5$ \\
\textsc{Selector} & --- & $\mathbf{+0.2}$ & $+8.0$ & $+7.8$ & $+10.3$ & $+8.3$ & $+23.0$ & $+14.5$ & $+11.5$ & $+12.4$ & $+35.6$ \\
\bottomrule
\end{tabular}}
\end{center}
\end{table}

\paragraph{Per-task results.}
Table~\ref{tab:pertask} lists, for every task with $H\ge2$, the official
MASE of one-shot decoding, of the deployed selector, and of the two
ground-truth-guided searches of Section~\ref{sec:dp}: the
prefix-retaining surrogate (DP oracle) and exhaustive search over all
$2^{H-1}$ rollouts (full oracle). 
Both oracles select with access to the ground truth and serve as oracle performance references for our learned selection procedure.
The full oracle is computed for the 58 tasks with $H\le11$; exhaustive
enumeration is infeasible at $H{=}18$.

{\small
\setlength{\tabcolsep}{5pt}
\begin{longtable}{lrr|rrr}
\caption{Per-task official MASE on the 66 tasks with $H\ge2$ ($H$ in
patches). Lower is better; the oracle columns use ground-truth
selection. Values are shown to five decimals so that every deviation
of the selector from one-shot decoding is visible: the selector value
is in bold where it improves on one-shot, a tie in
Table~\ref{tab:main} is an exact equality of the two columns, and the
W/T/L counts there can be read off directly.}
\label{tab:pertask}\\[0.6em]
\toprule
task & $H$ & one-shot & selector & DP oracle & full oracle \\
\midrule
\endfirsthead
\toprule
task & $H$ & one-shot & selector & DP oracle & full oracle \\
\midrule
\endhead
\midrule
\endfoot
\bottomrule
\endlastfoot
bizitobs\_l2c\_5T & 18 & 0.51692 & $\mathbf{0.51124}$ & 0.47539 & --- \\
boomlet\_1062 & 18 & 0.67769 & 0.67769 & 0.67639 & --- \\
boomlet\_1209 & 18 & 0.79877 & $\mathbf{0.79681}$ & 0.75308 & --- \\
boomlet\_1230 & 18 & 1.37049 & 1.37049 & 1.28606 & --- \\
boomlet\_1487 & 18 & 0.51066 & 0.51087 & 0.47699 & --- \\
kdd\_cup\_2022\_10T & 18 & 0.51992 & 0.52002 & 0.45892 & --- \\
LOOP\_SEATTLE\_5T & 18 & 0.68120 & $\mathbf{0.67968}$ & 0.54592 & --- \\
redset\_5T & 18 & 0.78646 & $\mathbf{0.78501}$ & 0.75097 & --- \\
entsoe\_1H & 11 & 0.53796 & $\mathbf{0.51843}$ & 0.45525 & 0.45053 \\
ercot\_1H & 11 & 1.35753 & $\mathbf{1.34969}$ & 1.23670 & 1.21879 \\
ETT\_1H & 11 & 1.12596 & 1.12596 & 1.08755 & 1.08218 \\
LOOP\_SEATTLE\_1H & 11 & 0.79393 & $\mathbf{0.79384}$ & 0.76101 & 0.75875 \\
M\_DENSE\_1H & 11 & 0.71320 & $\mathbf{0.71250}$ & 0.67841 & 0.67442 \\
proenfo\_gfc12 & 11 & 0.81200 & $\mathbf{0.79233}$ & 0.72557 & 0.72277 \\
proenfo\_gfc14 & 11 & 0.53869 & $\mathbf{0.52413}$ & 0.44547 & 0.44338 \\
proenfo\_gfc17 & 11 & 0.61039 & $\mathbf{0.59899}$ & 0.46396 & 0.45733 \\
SZ\_TAXI\_1H & 11 & 0.48788 & 0.48793 & 0.47630 & 0.47496 \\
uci\_air\_quality\_1H & 11 & 1.02060 & 1.03514 & 0.95585 & 0.95125 \\
jena\_weather\_10T & 9 & 0.43091 & $\mathbf{0.42552}$ & 0.39094 & 0.38737 \\
boomlet\_1631 & 6 & 0.71304 & 0.71309 & 0.70575 & 0.70501 \\
boomlet\_1676 & 6 & 0.70742 & 0.70742 & 0.70014 & 0.69989 \\
entsoe\_15T & 6 & 0.58461 & $\mathbf{0.57875}$ & 0.49605 & 0.49524 \\
entsoe\_30T & 6 & 0.54382 & $\mathbf{0.53787}$ & 0.48130 & 0.47818 \\
ETT\_15T & 6 & 0.69486 & $\mathbf{0.69228}$ & 0.66715 & 0.66371 \\
kdd\_cup\_2022\_30T & 6 & 0.54194 & 0.54203 & 0.51552 & 0.51411 \\
redset\_15T & 6 & 0.93209 & $\mathbf{0.92267}$ & 0.88923 & 0.88809 \\
solar\_with\_weather\_15T & 6 & 0.86410 & $\mathbf{0.84079}$ & 0.76716 & 0.76678 \\
SZ\_TAXI\_15T & 6 & 0.49999 & 0.50023 & 0.49222 & 0.49137 \\
boomlet\_1225 & 4 & 0.23225 & 0.23225 & 0.23223 & 0.23223 \\
boomlet\_1282 & 4 & 0.52286 & 0.52500 & 0.49606 & 0.49556 \\
boomlet\_285 & 4 & 0.35877 & $\mathbf{0.35757}$ & 0.35592 & 0.35584 \\
boomlet\_619 & 4 & 0.42342 & 0.42372 & 0.41592 & 0.41562 \\
boomlet\_772 & 4 & 0.33131 & 0.33131 & 0.33045 & 0.33045 \\
boomlet\_963 & 4 & 0.83903 & 0.83903 & 0.83162 & 0.82263 \\
hermes & 4 & 0.77611 & $\mathbf{0.77484}$ & 0.74209 & 0.74162 \\
rohlik\_orders\_1D & 4 & 1.18691 & 1.19121 & 1.10323 & 1.10196 \\
rossmann\_1D & 3 & 0.35576 & 0.35712 & 0.34481 & 0.34477 \\
walmart & 3 & 0.81674 & 0.81733 & 0.78768 & 0.78746 \\
bizitobs\_l2c\_1H & 2 & 0.38412 & 0.38591 & 0.38134 & 0.38134 \\
boomlet\_1855 & 2 & 0.53818 & 0.53818 & 0.53744 & 0.53744 \\
boomlet\_1975 & 2 & 0.16074 & $\mathbf{0.16026}$ & 0.15884 & 0.15884 \\
boomlet\_2187 & 2 & 0.81185 & 0.81185 & 0.81084 & 0.81084 \\
epf\_be & 2 & 0.64703 & 0.64703 & 0.62345 & 0.62345 \\
epf\_de & 2 & 0.60423 & 0.60423 & 0.54599 & 0.54599 \\
epf\_fr & 2 & 0.45006 & 0.45006 & 0.42580 & 0.42580 \\
epf\_np & 2 & 0.85019 & 0.85019 & 0.77631 & 0.77631 \\
epf\_pjm & 2 & 0.47190 & 0.47190 & 0.46592 & 0.46592 \\
ercot\_1D & 2 & 1.13827 & $\mathbf{1.09178}$ & 1.08174 & 1.08174 \\
ETT\_1D & 2 & 1.34138 & 1.34501 & 1.33291 & 1.33291 \\
favorita\_stores\_1D & 2 & 1.14020 & $\mathbf{1.14003}$ & 1.12550 & 1.12550 \\
favorita\_transactions\_1D & 2 & 0.84937 & $\mathbf{0.84880}$ & 0.83052 & 0.83052 \\
hierarchical\_sales\_1D & 2 & 0.68895 & $\mathbf{0.68890}$ & 0.68465 & 0.68465 \\
hospital\_admissions\_1D & 2 & 0.71705 & 0.71708 & 0.71651 & 0.71651 \\
jena\_weather\_1D & 2 & 1.40356 & 1.40356 & 1.39743 & 1.39743 \\
jena\_weather\_1H & 2 & 0.43739 & 0.43739 & 0.43307 & 0.43307 \\
LOOP\_SEATTLE\_1D & 2 & 0.95995 & 0.96000 & 0.95058 & 0.95058 \\
m5\_1D & 2 & 0.87988 & $\mathbf{0.87792}$ & 0.87224 & 0.87224 \\
M\_DENSE\_1D & 2 & 0.77215 & 0.77816 & 0.75953 & 0.75953 \\
redset\_1H & 2 & 1.50300 & $\mathbf{1.50157}$ & 1.48424 & 1.48424 \\
restaurant & 2 & 0.85991 & 0.86043 & 0.85366 & 0.85366 \\
solar\_1D & 2 & 0.75647 & $\mathbf{0.75645}$ & 0.75363 & 0.75363 \\
solar\_with\_weather\_1H & 2 & 1.02015 & $\mathbf{1.01945}$ & 0.98984 & 0.98984 \\
uci\_air\_quality\_1D & 2 & 1.31321 & 1.31321 & 1.29274 & 1.29274 \\
uk\_covid\_nation\_1D & 2 & 2.50044 & 2.52344 & 2.45113 & 2.45113 \\
uk\_covid\_nation\_1D\_cum & 2 & 9.44670 & 9.48487 & 8.96148 & 8.96148 \\
uk\_covid\_utla\_1D & 2 & 4.54598 & $\mathbf{4.54473}$ & 4.46456 & 4.46456 \\
\end{longtable}}

\end{document}